\documentclass{article}

\usepackage{microtype}
\usepackage{graphicx}
\usepackage{subfigure}
\usepackage{booktabs} 
\usepackage{amsmath}
\usepackage{amssymb}
\usepackage{xcolor}
\usepackage{hyperref}
\usepackage{bbm}
\usepackage{amsthm}

\usepackage{hyperref}

\usepackage[accepted]{icml2020}

\usepackage{amsmath, amsthm, amssymb,dsfont}

\newcommand{\U}{\mathcal{U}}
\newcommand{\V}{\mathcal{V}}
\newtheorem{theorem}{Theorem}

\newtheorem{definition}[theorem]{Definition}
\newtheorem{observation}[theorem]{Observation}
\newtheorem{remark}[theorem]{Remark}

   \newcommand{\reals}{\mathbb{R}}
   \newcommand{\naturals}{\mathbb{N}}
   \newcommand{\Ex}{\mathbb{E}}
   \renewcommand{\Pr}{\mathbb{P}}
   \newcommand{\Lo}[1]{{\mathcal L_{#1}}}
   \newcommand{\bLo}[1]{{\mathcal L^{0/1}_{#1}}}
   \newcommand{\rLo}[2]{{\mathcal L^{#1}_{#2}}}
   \newcommand{\lo}{\ell}
   \newcommand{\blo}{\ell^{0/1}}
   \newcommand{\rlo}[1]{\ell^{#1}}

    \newcommand{\indct}[1]{\mathds{1}\left[{#1}\right]}
   \newcommand{\B}{{\mathcal B}}
   \renewcommand{\P}{{\mathcal P}}
   \newcommand{\A}{{\mathcal A}}
   \newcommand{\Ocal}{{\mathcal O}}
   
  \newcommand{\F}{{\mathcal F}}
  \newcommand{\G}{{\mathcal G}}
  \renewcommand{\H}{{\mathcal H}}
  \newcommand{\C}{{\mathcal C}}

  \newcommand{\Ucal}{{\mathcal U}}

  \newcommand{\Hcal}{{\mathcal H}}

  \newcommand{\Mcal}{{\mathcal M}}
  
  \newcommand{\vc}{\mathrm{VC}}

  \renewcommand{\d}{\mathrm{dist}}
  \newcommand{\mar}[2]{\mathrm{mar}_{#2}^{#1}}
  \newcommand{\mH}{{\mathcal H_{\mar{}{}}^{\U}}}
  \newcommand{\err}[1]{\mathrm{err}_{#1}}
  \newcommand{\iid}{i.i.d.~}
   \newcommand{\supp}{\mathrm{supp}}

\icmltitlerunning{Black-box Certification and Learning under Adversarial Perturbations}

\begin{document}

\twocolumn[
\icmltitle{Black-box Certification and Learning under Adversarial Perturbations}

\icmlsetsymbol{equal}{*}

\begin{icmlauthorlist}
\icmlauthor{Hassan Ashtiani}{equal,mac}
\icmlauthor{Vinayak Pathak}{equal,scotia}
\icmlauthor{Ruth Urner}{equal,york}

\end{icmlauthorlist}

\icmlaffiliation{mac}{Department of Computing and Software, McMaster University, Hamilton, ON, Canada}
\icmlaffiliation{scotia}{Scotiabank, Toronto, ON, Canada}
\icmlaffiliation{york}{Lassonde School of Engineering, EECS Department, York University, Toronto, ON, Canada}

\icmlcorrespondingauthor{Hassan Ashtiani}{zokaeiam@mcmaster.ca}
\icmlcorrespondingauthor{Vinayak Pathak}{vpathak@uwaterloo.ca}
\icmlcorrespondingauthor{Ruth Urner}{ruth@eecs.yorku.ca}

\icmlkeywords{Adversarial learning, certification, sample complexity, query complexity}

\vskip 0.3in
]



\printAffiliationsAndNotice{\icmlEqualContribution} 

\begin{abstract}
We formally study the problem of classification under adversarial perturbations from a learner's perspective as well as a third-party who aims at certifying the robustness of a given black-box classifier. 
We analyze a PAC-type framework of semi-supervised learning and identify possibility and impossibility results for proper learning of VC-classes in this setting.
We further introduce a new setting of black-box certification under limited query budget, and analyze this for various classes of predictors and perturbation. 
We also consider the viewpoint of a black-box adversary that aims at finding adversarial examples, showing that the existence of an adversary with polynomial query complexity can imply the existence of a sample efficient robust learner.
\end{abstract}

\section{Introduction}
\label{s:intro}

We formally study the problem of classification under \emph{adversarial perturbations}.
An adversarial perturbation is an imperceptible alteration of a classifier's input which changes its prediction.
The existence of adversarial perturbations for real-world input instances and typical classifiers~\citep{SzegedyZSBEGF13} has contributed to a lack of trust in predictive tools derived from automated learning. 
Recent years have thus seen a surge of studies proposing various heuristics to enhance robustness to adversarial attacks \citep{ChakrabortySurvey2018arxiv}.
Existing solutions often either (i) modify the learning procedure to increase the adversarial robustness, e.g. by modifying the training data or the loss function used for training~\citep{SinhaND18,CohenRK19,SalmanLRZZBY19}, or (ii) post-process an existing classifier to enhance its robustness \citep{CohenRK19}.

A user of a predictive tool, however, may not oftentimes be involved in the training of the classifier nor have the technical access or capabilities to modify its input/output behavior.
Instead, the predictor may have been provided by a third party and the user may have merely a \emph{black-box access} to the predictor. 
That is, the predictor $h$ presents itself as an oracle that takes input \emph{query} $x$ and responds with the label $h(x)$.
The provider of the predictive tool, while not necessarily assumed to have malicious intent, is still naturally  considered untrusted, and the user thus has an interest in verifying the predictor's performance (including adversarial robustness) on its own application domain. 
While the standard notion of classification accuracy can be easily estimated from an i.i.d. sample generated from user's data generating distribution, estimating the expected robust loss is not that easy:
Given a labeled instance $(x,y)$, the user can immediately verify whether the instance is misclassified ($h(x)\neq y$) using a single query to $h$, but understanding whether $x$ is vulnerable under adversarial perturbations may require many more queries to the oracle.

We introduce and analyze a formal model for \emph{black-box certification} under query access, and provide examples of hypothesis classes and perturbation types\footnote{A perturbation type captures the set of admissible perturbations the adversary is allowed to make at each point (See Section~\ref{s:setup}).} for which such a certifier exists. We further introduce the notion of \emph{witness sets} for certification, and identify more general classes of problems and perturbation types that admit black-box certification with finite queries. 
On the contrary, we demonstrate cases of simple classes where the query complexity of certification is unbounded.

We further look at the problem from the viewpoint of the adversary, connecting the query complexity of an adversary (for finding adversarial examples) and that of the certifier.
An intriguing question that we explore is whether the sample complexity of learning a robust classifier with respect to a hypothesis class is related to the query complexity of an optimal adversary (or certifier) for that class. We uncover such a connection, showing that the existence of a successful adversary with polynomial query complexity for a properly compressible class implies sample efficient robust learnability of that class. For this, we adapt a compression-based argument, demonstrating a sample complexity upper bound for robust learning that is smaller than what was previously known~\cite{MontasserHS19} (assuming that a linear adversary exists and the class is properly compressible).

We start our investigations with the problem of robustly (PAC-)learning classes of finite VC-dimension.
It has been shown recently that, while the VC-dimension characterizes the proper learnability of a hypothesis class under the binary (misclassification) loss, there are classes of small VC-dimension that are not properly learnable under the robust loss \cite{MontasserHS19}.
We define the notion of the margin class
(associated with a hypothesis class and a perturbation type) and show that, if both the class and the margin class are simple (measured by their VC-dimension), then proper learning under robust loss is not significantly more difficult than learning with respect to the binary loss.

The corresponding complexity of the margin class, however, can be potentially large for specific choices of perturbation types and hypothesis classes. We thus investigate and provide scenarios where a form of semi-supervised learning can overcome the impossibility of proper robust learning.
We believe our investigations of robust learnability in these scenarios may help shed some light on where the difficulty of general robust classification stems from.

\subsection{Related work}
\label{ss:related_work}
Recent years have produced a surge of work on adversarial attack (and defense) mechanisms \citep{madry2017towards, ChakrabortySurvey2018arxiv, chen2017zoo, dong2018boosting, narodytska2017simple, papernot2017practical, akhtar2018threat, su2019one}, as well as the development of reference implementations of these \citep{GoodfellowMP18}. 
Here, we briefly review some earlier work on theoretical understanding of the problem. 

Several recent studies have suggested and analyzed approaches of training under data augmentation \cite{SinhaND18,SalmanLRZZBY19}. The general idea is to add adversarial perturbations to data points already at training time to promote smoothness around the support of the data generating distribution. These studies then provide statistical guarantees for the robustness of the learned classifier. 
Similarly, statistical guarantees have been presented for robust training that modifies the loss function rather than the training data \citep{WongK18}.
However, the notion of robustness certification used in these is different from what we propose. While they focus on designing learning methods that are certifiably robust, we aim at certifying an \emph{arbitrary} classifier and for a potentially new distribution.

The robust learnability of finite VC-classes has been studied only recently, often with pessimistic conclusions. An early result demonstrated that there exist distributions where robust learning requires provably more data than its non-robust counterpart~\citep{schmidt2018adversarially}.
Recent works have studied adversarially robust classification in the PAC-learning framework of computational learning theory~\citep{cullina2018pac, awasthi2019robustness, montasser2020efficiently} and presented hardness results for binary distribution and hypothesis classes in this framework \citep{diochnos2018adversarial, GourdeauKK019, diochnos2019arxiv}. On the other hand, robust learning has been shown to be possible, for example when the hypothesis class is finite and the adversary has a finite number of options for corrupting the input~\citep{feige2015learning}. This result has also been extended to the more general case of classes with finite VC-dimension~\citep{attias2018improved}.
It has also been shown that robust learning is possible (by Robust Empirical Risk Minimization (RERM)) under a feasibility assumption on the distribution and bounded covering numbers of the hypothesis class \citep{BubeckLPR19}. However, more recent work has presented classes of VC-dimension $1$, where the robust loss class has arbitrarily large VC-dimension~\citep{cullina2018pac} and, moreover, where proper learning (such as RERM) is impossible in a distribution-free finite sample regime~\citep{MontasserHS19}. 
Remarkably, the latter work also presents an improper learning scheme for any VC-class and any adversary type. The sample complexity of this approach, however, depends on the dual VC-dimension which can be exponential in the VC-dimension of the class.

We note that two additional aspects of our work have appeared in the the literature before: considering robust learnability by imposing computational constraints on an adversary has been explored recently \citep{BubeckLPR19, GourdeauKK019, Garg2019arxiv}.
Earlier work has also hypothesized that unlabeled data may facilitate adversarially robust learning, and demonstrated a scenario where access to unlabeled data yields a better bound on the sample complexity under a specific data generative model \citep{CarmonRSDL19, stanforth2019labels}.

Less closely related to our work, the theory of adversarially robust learnability has been studied for non-parametric learners. A first study in that framework showed that a nearest neighbor classifier's robust loss converges to that of the Bayes optimal \citep{WangJC18}. A follow-up work then derived a characterization of the best classifier with respect to the robust loss (analogous to the notion of the Bayes optimal), and suggested a training data pruning approach for non-parametric robust classification \citep{Kamalika2019arxiv}.

\subsection{Outline and summary of contributions}
\label{ss:contributions}

{\bf Problem setup and the adversarial loss formulation.}
In Section~\ref{s:setup}, we provide the formal setup for the problem of adversarial learning. We also decompose the adversarial loss, and define the notion of the margin class associated with a hypothesis class and a perturbation type (Def.~\ref{def:margin_class}).

{\bf Using unlabeled data for adversarial learning of VC-classes.}
In Section~\ref{s:vc_classes}, we study the sample complexity of \emph{proper} robust learning. While this sample complexity can be infinite for general VC-classes~\citep{cullina2018pac, MontasserHS19, YinRB19}, we show that VC-classes are properly robustly learnable if the margin class also has finite VC-dim (Thm.~\ref{thm:learnability_pos}).
We formalize an idealized notion of semi-supervised learning where the learner has additional oracle access to probability weights of the margin sets. We show that, perhaps counter intuitively, oracle access to both (exact) margin weights and (exact) binary losses, does not suffice for identifying the minimizer of the adversarial loss in a class $\H$ (Thm.~\ref{thm:stat_impossibility}), even in the $0/1$-realizable case (Thm.~\ref{thm:double}). 
However, under the additional assumption of robust realizability,  proper learning becomes feasible with access to the marginal or sufficient unlabled data (Thms.~\ref{thm:robust_realizable} and \ref{thm:robust_realizable_unlabeled}).

{\bf Black-box certification with query access.}
We formally define the problem of \emph{black-box certification through query access} (Def.~\ref{def:certification}), and demonstrate examples where certification is possible (Obs.~\ref{obs:half-spaces}) or impossible (Obs.~\ref{obs:lower-bound}).
Motivated by this impossibility result, we also introduce a \emph{tolerant} notion of certification (Def.~\ref{def:Tolerant}). We show that while more classes are certifiable with this definition (Obs.~\ref{obs:tolerant}), some simple classes remain impossible to certify (Obs.~\ref{obs:tolerant_impossibility}).
We identify a sufficient condition for certifiability of a hypothesis class w.r.t. a perturbation type through the notion of \emph{witness sets} (Def.~\ref{def:wintness} and Thm.~\ref{thm:partial_order}).
We then consider the query complexity of the adversary (as opposed to that of the certifier) for finding adversarial instances (Def.~\ref{def:perfect_adversary}, \ref{def:efficient_adversary}, and \ref{def:non-adaptive_adversary}) and---for the case of a non-adaptive adversary---relate it to the existence of a witness set (Obs.~\ref{obs:witness_vs_non-adaptive_adversary}).
    
{\bf Connecting adversarial query complexity and PAC-learnability.}
The culminating result connects the two themes of our work: robust (PAC-)learnability, and query complexity of an adversary. With Theorem~\ref{thm:bounded_adv_impl_learnable}, we show that existence of a perfect adversary with small query complexity implies sample-efficient robust learning for properly compressible classes.

We include the proof sketches in the paper, and refer the reader to the supplementary material for detailed proofs.

\section{Setup and Definitions}\label{s:setup}
We let $X$ denote the domain (often $X\subseteq\reals^d$) and $Y$ (mostly $Y=\{0,1\}$) a (binary) label space.
We assume that data is generated by some distribution $P$ over $X\times Y$ and let $P_X$ denote the marginal of $P$ over $X$.
A \emph{hypothesis} is a function $h:X\to Y$, and can naturally be identified with a subset of $X\times Y$, namely $h = \{(x,y)\in X\times Y ~\mid~ x\in X,~y = f(x)\}$.
Since we are working with binary labels, we also sometimes identify a hypothesis $h$ with the pre-image of $1$ under $h$, that is the domain subset $\{x\in X \mid h(x) = 1\}$.
We let $\F$ denote the set of all Borel functions\footnote{For an uncountable domain, we only consider Borel-measurable hypotheses to avoid dealing with measurability issues.} from $X$ to $Y$ (or all functions in case of a countable domain). A \emph{hypothesis class} is a subset of $\F$, often denoted by $\H\subseteq \F$.

The quality of prediction of a hypothesis on a labeled example $(x,y)$ is measured by a \emph{loss function} $\lo:(\F\times X \times Y) \to \reals$.
For classification problems, the quality of prediction is typically measured with the \emph{binary} loss 
$$
\blo(h, x, y) = \indct{h(x) \neq y} 
$$
, where $\indct{\alpha}$ denotes the indicator function for predicate $\alpha$.
For (adversarially) robust classification, we let $\U:X\to 2^X$, the \emph{perturbation type}, be a function that maps each instance to the set of admissible perturbations at point $x$.
We assume that the perturbation type satisfies $x\in \U(x)$ for all $x\in X$.
If $X$ is equipped with a metric $\d$, then a natural choice for the set of perturbations at $x$ is a ball $\B_r(x) = \{z\in X ~\mid~ \d(x,z) \leq r\}$ of radius $r$ around $x$. 
For an $x\in X$ and $h\in\H$, we say that $x'\in\U(x)$ is an \emph{adversarial} point of $x$ with respect to $h$ if $h(x)\neq h(x')$.
We use the following definition of the adversarially robust loss with respect to perturbation type $\U$
$$
 \rlo{\U}(h, x, y) = \indct{\exists z\in \U(x) ~:~ h(z) \neq y}.
$$

If $\U(x)$ is always a ball of radius $r$ around $x$, we will also use the notation $\rlo{r}(h, x, y) = \rlo{\B_r}(h, x, y)$.
We assume that the perturbation type is so that $\rlo{\U}(f, \cdot, \cdot)$ is a measurable function for all $f\in \F$. A sufficient condition for this is that the set $\U(x)$ are open sets (where $X$ is assumed to be equipped with some topology) and the pertubation type further satisfies $z\in \U(x)$ if and only if $x\in \U(z)$ for all $x,z\in X$ (see Appendix 
\ref{app:measurable} for a proof and an example of a simple perturbation type that renders the the corresponding loss function of a threshold predictor non-measurable).

We denote the \emph{expected loss} (or \emph{true loss}) of a hypothesis $h$ with respect to the distribution $P$ and loss function $\lo$ by $\Lo{P} (h) = \Ex_{(x,y)\sim P} [\lo(h , x, y)]$. 
In particular, we will denote the true binary loss by $\bLo{P}(h)$ and the true robust loss by $\rLo{\U}{P}(h)$.
Further, we denote the \emph{approximation error} of class $\H$ with respect to distribution $P$ and loss function $\lo$ by
$\Lo{P}(\H) = \inf_{h\in \Hcal} \Lo{P}(h).$

The \emph{empirical loss} of a hypothesis $h$ with respect to loss function $\lo$ and a sample $S = ((x_1, y_1), \ldots, (x_n, y_n))$ is defined as $\Lo{S}(h) = \frac{1}{n}\sum_{i=1}^n \lo(h, x_i, y_i)$.

A \emph{learner} $\A$ is a function that takes in a finite sequence of labeled instances $S = ((x_1, y_1), \ldots, (x_n, y_n))$ and outputs a hypothesis $h = \A(S)$. The following is a standard notion of (PAC-)learnability from finite samples of a hypothesis class \cite{vapnikcherv71, Valiant84, blumer1989learnability, shalev2014understanding}. 

\begin{definition}[(Agnostic) Learnability]\label{def:learn}
A hypothesis class $\Hcal$ is agnostic learnable with respect to set of distributions $\P$ and loss function $\lo$, if there exists a learner $\A$ such that for all $\epsilon,\delta \in (0,1)$, there is a sample size $m(\epsilon, \delta)$ such that,  for any distribution $P\in\P$, if the input to $\A$ is an iid sample $S$ from $P$ of size $m \geq m(\epsilon, \delta)$, then, with probability at least $(1-\delta)$ over the samples, the learner outputs a hypothesis $h = \A(S)$ with
$\Lo{P}(h) \leq \Lo{P}(\H) + \epsilon.$

$\Hcal$ is said to be \emph{learnable in the realizable case} with respect to loss function $\lo$, if the above holds under the condition that $\Lo{P}(\H) = 0$.
We say that $\H$ is \emph{distribution-free learnable} (or simply \emph{learnable}) if it is learnable when $\P$ is the set of all probability measures over $X\times Y$.
\end{definition}

\begin{definition}[VC-dimension]\label{def:vc}
 We say that a collection of subsets $\G\subseteq 2^X$ of some domain $X$ \emph{shatters} a subset $B\subseteq X$ if for every $F \subseteq B$ there exists $G\in \G$ such that  $G\cap B = F$. The \emph{VC-dimension} of $\G$, denoted by $\vc(\G)$, is defined to be the supremum of the size of the sets that are shattered by $\G$. 
\end{definition}

It is easy to see that the VC-dimension of a binary hypothesis class $\H$ is independent of whether we view $\H$ as a subset of $X\times Y$ or pre-images of $1$ (thus, subsets of $X$).
It is well known that, for the binary loss, a hypothesis class is (distribution-free) learnable if and only if it has finite VC-dimension~\citep{blumer1989learnability}. 
Furthermore, any learnable binary hypothesis class can be learned with a \emph{proper learner}.

\begin{definition}[Proper Learnability]\label{def:proper_learn}
 We call a learner $\A$ a \emph{proper learner} for the class $\H$ if, for all input samples $S$, we have $\A(S)\in \H$.
 A class $\H$ is \emph{properly 
 learnable} if the conditions in Definition \ref{def:learn} hold with a proper learner $\A$.
\end{definition}

It has recently been  shown that there are classes of finite VC-dimension that are not properly learnable with respect to the adversarially robust loss \citep{MontasserHS19}.

\subsection{Decomposing the robust loss}
\label{sec:decompose}
In this work, we adapt the most commonly used notion of a adversarially robust loss \citep{MontasserHS19,Kamalika2019arxiv}.
Note that, we have $\rlo{\U}(h, x, y) = 1$ if and only if at least one of the following conditions holds:\\
 $\bullet$~ $h$ makes a mistake on $x$ with respect to label $y$, or\\
 $\bullet$~ there is a close-by instance $z\in \U(x)$ that $h$ labels different than $x$, that is, $x$ is close to $h$'s decision boundary.

The first condition holds when $(x,y)$ falls into the \emph{error region}, $\err{h} = (X\times Y)\setminus h$. The notion of error region then naturally captures the (non-adversarial) loss:
$$\bLo{P} (h) = \Pr_{(x,y)\sim P} [(x,y)\in  \err{h}] = P(\err{h}).$$
The second condition holds when $x$ lies in the \emph{margin area} of $h$. The following definition makes this notion explicit.

Let $h\in\F$ be some hypothesis. We define the \emph{margin area} of $h$ with respect to perturbation type $\U$, as the subset $\mar{\U}{h}\subseteq X\times Y$  defined by
\[
\mar{\U}{h} =  \{(x,y)\in X\times Y ~\mid~ \exists z\in\U(x): h(x)\neq h(z)\}
\]

Based on these definitions, the adversarially robust loss with respect to $\U$ is $1$ if and only if the sample $(x,y)$ falls into the error region $\err{h}$ and/or the margin area $\mar{\U}{h}$ of $h$:
$$\rLo{\U}{P} (h) 
= P(\err{h}\cup\mar{\U}{h}).$$

\begin{definition}\label{def:margin_class}
For class $\H$, we refer to the collection $\mH = \{ \mar{\U}{h} ~\mid~ h\in\H\}$ as the \emph{margin class} of $\H$.
\end{definition}

While we defined that margin areas $\mar{\U}{h}$ as subsets of $X\times Y$, it is sometimes natural to identify them with their projection on $X$, thus simply as subsets of $X$. 

\begin{remark}
There is more than one way to formulate a loss function that captures both classification accuracy and robustness to (small) adversarial perturbations. The notion we adopt has the property that even the true labeling function can have positive robust loss, if the true labels themselves change within the adversarial neighbourhoods. A natural alternative is to say an adversarial point is a point in the neighbourhood of an instance that is misclassified by the classifier.
However, note that such a notion cannot be phrased as a loss function $\ell(h, x, y)$ (as it depends on the true label of the perturbed instance).
Previous studies have provided excellent discussions of the various options
\citep{diochnos2018adversarial,GourdeauKK019}.
\end{remark}

\paragraph{Semi-Supervised Learning (SSL)}
Since the margin areas $\mar{\U}{h}$ can naturally be viewed as subsets of $X$, their weights $P(\mar{\U}{h})$ under the data generating distribution can potentially be estimated with samples from $P_X$, that is, from \emph{unlabeled data}.
A learner that takes in both a labeled sample $S$ from $P$ and an unlabeled sample $T$ from $P_X$, is called a \emph{semi-supervised learner}.
For scenarios where robust learning has been shown to be hard, we explore whether this hardness can be overcome by SSL.
We consider semi-supervised learners that take in labeled and unlabeled samples, and also \emph{idealized semi-supervised learners} that, in addition to a labeled samples have oracle access to probability weights of certain subsets of $X$ \cite{GopfertBBGTU19}.

\section{Robust Learning of VC Classes}
\label{s:vc_classes}
It has been shown that there is a class $\H$ of bounded VC-dimension ($\vc(\H) = 1$ in fact) and a perturbation type $\U$ such that $\H$ is not robustly properly learnable \cite{MontasserHS19}, even if the distribution is realizable with respect to $\H$ under the $\U$-robust loss. The perturbation type $\U$ in that lower bound construction can actually chosen to be balls with respect to some metric over $X = \reals^d$ (for any $d$, even $d=1$).
The same work also shows that if a class has bounded VC-dimension, then it is (improperly) robustly learnable with respect to any perturbation type $\U$.

\begin{theorem}[\cite{MontasserHS19}]
  (1) There is a class $\H$ over $X = \reals^d$ with $\vc(\H) = 1$, and a set of distributions $\P$ with $\rLo{r}{P}(\H) = 0$ for all $P\in \P$, such that $\H$ is not proper learnable over $\P$ with respect to loss function $\rlo{r}$. \\
  (2) Let $X$ be any domain and $\U:X \to 2^X$ be any type of perturbation, and let $\H\subseteq \{0,1\}^X$ be a hypothesis class with finite VC-dimension. Then $\H$ is distribution-free agnostic learnable with respect to loss function $\rlo{\U}$.
\end{theorem}

While the second part of the above theorem seems to settle adversarially robust learnability for binary hypothesis classes, the positive result is achieved with a compression-based learner, which has potentially much higher sample complexity than what suffices for the binary loss. In fact, the size of the best known general compression scheme~\citep{moran2016sample} depends on the VC-dimension of the dual class of $\H$, making the sample complexity of this approach generally exponential in VC-dimension of $\H$. 

We first show that the impossibility part of the above theorem crucially depends on the combination of the class $\H$ and a pertubation type $\U$ (despite these being balls in a Euclidian space) so that the margin class $\mH$ has infinite VC-dimension. We prove that, if both $\H$ and $\mH$ have finite VC-dimension then $\H$ is (distribution-free) learnable with respect to the robust loss, with a proper learner.

\begin{theorem}[Proper learnability for finite VC and finite margin-VC]\label{thm:learnability_pos}
Let $X$ be any domain and $\H \subseteq \F$ be a hypothesis class with finite VC-dimension. Further, let $\U: X \to 2^X$ be any perturbation type such that $\mH$ has finite VC-dimension. We set $D = \vc(\H) + \vc(\mH)$.
 Then $\H$ is distribution-free (agnostically) properly learnable with respect to the robust loss $\rlo{\U}$, and the sample complexity is 
 $O\left(\frac{D\log(D) + \log(1/\delta) }{\epsilon^2}\right).$
\end{theorem}

\begin{proof}[Proof Sketch]
We provide the more detailed argument in the appendix.
Recall that a set $S\subseteq X\times Y$ is said to be an $\epsilon$-approximation of $P$ with respect to $\H \subseteq X\times Y$ if
for all $h\in \H$ we have  $\left|P [h] - \frac{|h \cap S|}{|S|}\right|\leq \epsilon$,
that is, if the empirical estimates with respect to $S$ of the sets in $h$ are $\epsilon$-close to their true probability weights.
Consider the class of subsets  $\G = \{(\err{h}\cup\mar{\U}{h}) \subseteq X \times Y ~\mid~ h\in \H\}$ of point-wise unions of error and margin regions. A simple counting argument shows that $\vc(\G) \leq D\log(D)$, where $D = \vc(\H) + \vc(\mH)$.
 Thus, by basic VC-theory, a sample of size $\Theta\left(\frac{D\log D + \log(1/\delta) }{\epsilon^2}\right)$ will be an $\epsilon$-approximation of $\G$ with respect to $P$ with probability at least $1-\delta$. 
 Thus any empirical risk minimizer with respect to $\rlo{\U}$ is a successful proper and agnostic robust learner for $\H$.
\end{proof}

\begin{observation}
 We believe the conditions of Theorem \ref{thm:learnability_pos} hold for most natural classes and perturbation types $\U$.
 Eg.~if $\H$ is the class of linear predictors in $\reals^d$ and $\U$ are sets of balls with respect to some $\ell_p$-norm, then both $\H$ and $\mH$ have finite VC-dimension (see also \cite{YinRB19}).
\end{observation}

\subsection{Using unlabeled data for robust proper learning}
In light of the above two general results, we turn to investigate whether unlabeled data can help in overcoming the discrepancy between the two setups.
In particular, under various additional assumptions, we consider the case of $\vc(\H)$ being finite but $\vc(\mH)$ (potentially) being infinite and a learner having additional access to $P_X$.

We model knowledge of $P_X$ as the learner having access to an oracle that returns the probability weights of various subsets of $X$. We say that the learner has access to a \emph{margin oracle for class $\H$}
if, for every $h\in\H$, it has access (can query) the probability weight of the margin set of $h$, that is $P(\mar{\U}{h})$. Since the margin areas can be viewed as subsets of $X$, if the margin class of $\H$ under perturbation type $\U$ has finite VC-dimension, a margin oracle can be approximated using an unlabeled sample from the distribution $P$. 

Similarly, one could define an \emph{error oracle for $\H$} as an oracle, that, for every $h\in\H$ would return the weight of the error sets $P(\err{h})$. This is typically approximated with a labeled sample from the data-generating distribution, if the class has finite VC-dimension. This is similar to the  settings of learning by distances \cite{Ben-DavidIK95} or learning with statistical queries \cite{Kearns98, Feldman17}.

To minimize the adversarial loss however, the learner needs to find (through oracle access or through approximations by samples) a minimizer  of the weights $P(\err{h} \cup \mar{\U}{h})$. We now first show that having access to both an exact error oracle and an exact margin oracle does not suffice for this. 

\begin{theorem}\label{thm:stat_impossibility}
There is a class $\H$ with $\vc(\H) = 1$ over a domain $X$ with $|X| = 7$, a perturbation type $\U:X \to 2^X$, and two distributions $P^1$ and $P^2$ over $X\times \{0,1\}$, that are indistinguishable with error and margin oracles for $\H$, while their robust loss minimizers in $\H$ differ. 
\end{theorem}

\begin{proof}
Let $X = \{x_1, x_2,\ldots, x_7\}$ be the domain.
We consider two distributions $P^1$ and $P^2$ over $X\times \{0,1\}$.
Both have true label $0$ on all points, that is $P(y = 1 |x) = 0$ for all $x\in X$. However their marginals $P^1$ and $P^2$ differ:
\begin{align*}
P^1_X(x_1) & ~=~ P^1_X(x_3) = 0,~
P^1_X(x_2)  ~=~ 2/6, \text{and}\\
P^1_X(x_i) & ~=~ 1/6 ~\text{for}~ i\in\{4, 5, 6, 7\}.\\
P^2_X(x_4) & ~=~ P^2_X(x_6) = 0,~
P^2_X(x_5)  ~=~ 2/6, \text{and}\\
P^2_X(x_i) & ~=~ 1/6 ~\text{for}~ i\in\{1, 2, 3, 7\}.
\end{align*}

The class $\H$ consists of two functions: $h_1 = \indct{x = x_2 \lor x = x_3}$ and $h_2 = \indct{x = x_5 \lor x= x_6}$. Further, we consider the following perturbation sets (for readability, we first state them without the points themselves):
\begin{align*}
& \tilde{\U}(x_1) = \{x_2\},~ \tilde{\U}(x_2) = \{x_1, x_3\},~ \tilde{\U}(x_3) = \{x_2\},\\
& \tilde{\U}(x_4) = \{x_5\},~ \tilde{\U}(x_5) = \{x_4, x_6\},~ \tilde{\U}(x_6) = \{x_5\},\\
& \tilde{\U}(x_7) = \emptyset
\end{align*}
Now we set $\U(x_i) = \tilde{\U}(x_i)\cup \{x_i\}$, so that each point is included in its own perturbation set.
Now, both $h_1$ and $h_2$ have $0/1$-loss $2/6 = 1/3$ on both $P^1$ and $P^2$.
And for both $h_1$ and $h_2$ the margin areas have weight $2/6 = 1/3$ on both $P^1$ and $P^2$. However, the adversarial loss minimizer for $P^1$ is $h_1$ and for $P^2$ is $h_2$ (by a gap of $1/6$ each).
\end{proof}

While the impossibility result in the above example, of course, can be overcome by estimating the weights of the seven points in the domain, the construction exhibits that merely estimating classification error and weights of margin sets does not suffice for proper learning with respect to the adversarial loss. The example shows, that the learner also needs to take into account the interactions (intersections between the sets) of the two components of the adversarial loss.
However the weights of the intersection sets $\err{h}\cap\mar{\U}{h}$, inherently involve label information.

In the following subsection we show that realizability with respect to the robust loss implies that robust learning becomes possible with access to a (bounded size) labeled sample from the distribution and additional access to a margin oracle or a (bounded size) unlabeled sample. In the appendix Section \ref{app:add_results}, we further explore weakening this assumption to only require $0/1$-reazability with access to stronger version of the margin oracle.

\subsubsection{Robust realizability:\\ $\exists h^*\in \H$ with $\rLo{\U}{P}(h^*) = 0$}

This is the setup of the impossibility result for proper learning \cite{MontasserHS19}. We show that proper learning becomes possible with access to a margin oracle for $\H$.

\begin{theorem}\label{thm:robust_realizable}
Let $X$ be some domain, $\H$ a hypothesis class with finite VC-dimension and $\U:X\to 2^X$ any perturbation type.
If a learner is given additional access to a margin oracle for $\H$, then $\H$ is properly learnable with respect to the robust loss $\rlo{\U}$ and the class of distributions $P$ that are robust-realizable by $\H$, $\rLo{\U}{P}(\H) = 0$, with labeled sample complexity $\tilde{O}(\frac{\vc(\H) + \log(1/\delta)}{\epsilon})$ 
\end{theorem}
\begin{proof}[Proof Sketch]
By the robust realizability, there is an $h^*\in \H $ with $\rLo{\U}{P}(h^*) = 0$ implying that $\bLo{P}(h^*) = 0$, that is, the distribution is (standard) realizable by $\H$. Basic VC-theory tells us that an iid sample $S$ of size $\Tilde{\Theta}\left(\frac{\vc(\H)+ \log(\frac{1}{\delta})}{\epsilon}\right)$ guarantees that all functions in the \emph{version space} of $S$ (that is all $h$ with $\bLo{S}(h) = 0$) have true binary loss at most $\epsilon$ (with probability at least $1-\delta$). Now, with access to a margin oracle for $\H$ a learner can remove all hypotheses with $P(\mar{U}{h}) > 0$ from the version space and return any remaining hypothesis (at least $h^*$ will remain).
\end{proof}

Note that the above procedure crucially depends on actual access to a margin oracle. The weights $P(\mar{U}{h})$ cannot be generally estimated if $\mH$ has infinite VC-dimension, as the impossibility result for proper learning from finite samples shows. Thus proper learnability even under these (strong) assumptions cannot always be manifested by a semi-supervised proper learner that has access only to finite amounts of unlabeled data. 
We also note that the above result (even with access to $P_X$) does not allow for an extension to the agnostic case via the type of reductions known from compression-based bounds~\citep{MontasserHS19, moran2016sample}. 

On the other hand, if the margin class has finite, but potentially much larger VC-dimension than $\H$, then we can use unlabeled data to approximate the margin oracle in Theorem \ref{thm:robust_realizable}. The following result thus provides an improved bound on the number of \emph{labeled samples} that suffice for robust proper learning under the assumptions of Theorem \ref{thm:learnability_pos}.

\begin{theorem}\label{thm:robust_realizable_unlabeled}
Let $X$ be some domain, $\H$ a hypothesis class with finite VC-dimension and let $\U:X\to 2^X$ be a perturbation type such that the margin class $\mH$ also has finite VC-dimension.
If a learner is given additional access to an (unlabeled) sample $T$ from $P_X$, then $\H$ is properly learnable with respect to the robust loss $\rlo{\U}$ and the class of distributions $P$ that are robust-realizable by $\H$, $\rLo{\U}{P}(\H) = 0$, with labeled sample complexity $\tilde{O}(\frac{\vc(\H) + \log(1/\delta)}{\epsilon})$ and unlabeled sample complexity $\tilde{O}(\frac{\vc(\mH) + \log(1/\delta)}{\epsilon})$ 
\end{theorem}

\begin{proof}
The stated sample sizes imply that all functions in $\H$ in the version space of the labeled sample $S$ have true binary loss at most $\epsilon$ and all functions in $\H$ whose margin areas are not hit by $T$ have true margin weight at most $\epsilon$. The learner can thus output any function $h$ with $0$ classification error on $S$ and $0$ margin weight under $T$ (at least $h^*$ will satisfy these conditions), and we get $\rLo{\U}{h} = P(\err{h} \cup \mar{\U}{h}) \leq 2\epsilon$. 
\end{proof}

The assumption in the above theorems states that there exists one function $h^*$ in the class that has both perfect classification accuracy and no weight in its margin area. The proof of the impossibility construction of Theorem \ref{thm:stat_impossibility} employs a class and distributions where no function in the class has perfect margin or perfectly classifies the task. We can modify that construction to
show that the ``double realizability'' in Theorem \ref{thm:robust_realizable} is necessary if the access to the marginal should be restricted to a margin oracle for $\H$. The proof of the follwing result can be found in Appendix \ref{app:proofs}.

\begin{theorem}\label{thm:double}
There is a class $\H$ with $\vc(\H) = 1$ over a domain $X$ with $|X| = 8$, a perturbation type $\U:X \to 2^X$, and two distributions $P^1$ and $P^2$ over $X\times \{0,1\}$, such that there are functions $h_r, h_c \in \H$  with 
$\bLo{P^i}(h_r) = 0$ and $P^i(\mar{\U}{h_c}) = 0$
for both $i\in \{1,2\}$, while $P^1$ and $P^2$
are indistinguishable with error and margin oracles for $\H$ and their robust loss minimizers in $\H$ differ.
\end{theorem}

\section{Black-box Certification and the Query Complexity of Adversarial Attacks}\label{s:query}

Given a fixed hypothesis $h$, a basic concentration inequality (e.g., Hoeffding's inequality) indicates that the empirical loss of $h$ on a samples $S\sim P^m$, $\bLo{S}(h)$, gives an $O(m^{-1/2})$-accurate estimate of the true loss with respect to $P$, $\bLo{P}(h)$. In fact, in order to compute $\bLo{S}(h)$, we do not need to know $h$ directly; it would suffice to be able to query $h(x)$ on the given sample. Therefore, we can say it is possible to estimate the true binary loss of $h$ up to additive error $\epsilon$ using $O(1/{\epsilon^2})$ samples from $P$ and $O(1/{\epsilon^2})$ queries to $h(.)$. 

The high-level question that we ask in this section is whether and when we can do the same for the adversarial loss, $\rLo{\U}{P}(h)$. If possible, it would mean that we can have a third-party that ``certifies'' the robustness of a given black-box predictor (e.g., without relying on the knowledge of the learning algorithm that produced it)

\begin{definition}[Label Query Oracle]
We call an oracle $\Ocal_h$ a \emph{label query oracle for a hypothesis $h$}, if for all $x\in X$, upon querying for $x$, the oracle returns the label $\Ocal_h(x) = h(x)$.
\end{definition}

\begin{definition}[Query-based Algorithm]
We call an algorithm $\A: (\bigcup_{i=1}^{\infty} X^i, \Ocal_h) \to \reals$ a \emph{query-based algorithm}, if $\A$ has access to a label query oracle $\Ocal_h$.
\end{definition}

\begin{definition}[Certifiablility]\label{def:certification}
A class $\H$ is \emph{certifiable} with respect to $\U$ if there exists a query based algorithm $\A$ and there are functions $q,m: (0,1)^2 \to\naturals$ such that for every $\epsilon, \delta \in (0,1]$, every distribution $P$ over $X\times Y$, and every $h\in \H$, we have that with probability at least $1-\delta$ over an iid sample $S\sim P_X^m$ of size $m\geq m(\epsilon, \delta)$ 
\[
|\A(S, \Ocal_h) ~-~ \rLo{\U}{P}(h)|~<~\epsilon
\]
with a query budget of $q(\epsilon,\delta)$ for $\A$. In this case, we say that \emph{$\H$ admits $(m,q)$ blackbox query certification}.
\end{definition}

In light of Section \ref{sec:decompose}, the task of robust certification is to estimate the probability weight of the set $\err{h}\cup\mar{\U}{h}$.

\begin{observation}
\label{obs:half-spaces}
Let $\H$ be the set of all half-spaces in $\mathbb{R}^2$ and let $\U(x) = \{z: \|x-z\|_1\leq 1\}$ be the unit ball wrt $\ell_1$-norm centred at $x$. 
Then $\H$ admits $(m, q)$-certification under $\U$ for functions $m,q \in O(1/\epsilon^2)$.
\end{observation}

\begin{proof}
Say we have a sample $S\sim P_X^m$. For each point $x\in S$ define the set $w(x) = \{x+(0, 1), x+(1,0), x+(-1,0), x+(0,-1)\}$, i.e., the four corner points of $\U(x)$. The certifier can determine whether $x\in\err{h}$ by querying the label of $x$; further it can determine whether $x\in\mar{\U}{h}$ by querying all points in $w(x)$. Let $W = \cup_{x\in S} w(x)$. By querying all points in $S\cup W$, the certifier can calculate the robust loss of $h$ on $S$. This will be an $\epsilon$-accurate estimate of $L^\U_P(h)$ when $m = O(1/\epsilon^2)$.
\end{proof}

We immediately see that certification is non-trivial, in that there are cases where robust certification is impossible. The proof can be found in Appendix \ref{app:query_proofs}.

\begin{observation}
\label{obs:lower-bound}
Let $\H$ be the set of all half-spaces in $\mathbb{R}^2$ and let $\U(x) = \{z: \|x-z\|_2\leq 1\}$ be the unit ball wrt $\ell_2$-norm centred at $x$. Then $\H$ is not certifiable under $\U$.
\end{observation}

This motivates us to define a \emph{tolerant certification} version.
\begin{definition}[Restriction of a perturbation type]
Let $\U, \V: X\rightarrow 2^X$ be a perturbation types. We say that \emph{$\U$ is a restriction of $\V$} if $\U(x)\subseteq\V(x)$ for all $x\in X$.
\end{definition}

Note that if $\U$ is a restriction of $\V$, then, for all distributions $P$ and predictors $h$ we have $\rLo{\U}{P} \leq \rLo{\V}{P}$.

\begin{definition}[Tolerant Certification]\label{def:Tolerant}
A class $\H$ is \emph{tolerantly certifiable} with respect to $\U$ and $\V$, where $\U$ is a restriction of $\V$, if there exists a query based algorithm $\A$, and there are functions $q,m: (0,1)^2 \to\naturals$ such that for every $\epsilon, \delta \in (0,1]$, every distribution $P$ over $X \times Y$, and every $h\in \H$, we have that with probability at least $1-\delta$ over an \iid sample $S\sim P_X^m$ of size $m\geq m(\epsilon, \delta)$ 
\[
\A(S, \Ocal_h) ~\in~[\rLo{\U}{P}(h)-\epsilon, \rLo{\V}{P}(h)+\epsilon]
\]
with a query budget of $q(\epsilon,\delta)$ for $\A$. In this case, we say that \emph{$\H$ admits tolerant $(m,q)$ blackbox query certification}.
\end{definition}

\begin{observation}\label{obs:tolerant}
Let $\U(x) = \{z: \|x-z\|_2\leq 1\}$ and $\V(x) = \{z: \|x-z\|_2\leq 1+\gamma\}$. Let $\H$ be the set of all half-spaces in $\mathbb{R}^2$. Then $\H$ is $(O(1/\epsilon^2), O(1/\sqrt{\gamma}\epsilon^2))$ tolerantly certifiable with respect to $\U$ and $\V$.
\end{observation}

\begin{proof}[Proof sketch]
For each $x$, we can always find a regular polygon with $O(\pi/\sqrt{\gamma})$ vertices that ``sits'' between the $\U(x)$ and $\V(x)$. Therefore, in order to find out whether $x$ is adversarially vulnerable or not, it would suffice to make $O(\pi/\sqrt{\gamma})$ queries. Combining this with Hoeffding's inequality shows that if we sample $1/\epsilon^2$ points from $P$ and make $O(\pi/\sqrt{\gamma})$ queries for each, we can estimate $\rLo{\U,\V}{P}(h)$ within error $\epsilon$.
\end{proof}

Though more realistic, even the tolerant notion of certifiability does not make all seemingly simple classes certifiable.

\begin{observation}\label{obs:tolerant_impossibility}
Let $\U(x) = \{z: \|x-z\|^2\leq 1\}$ and $\V(x) = \{z: \|x-z\|^2\leq 1+\gamma\}$. There exists a hypothesis class $\H$ with VC-dimension 1, such that $\H$ is not tolerantly certifiable with respect to $\U$ and $\V$.
\end{observation}

\begin{proof}[Proof sketch]
For any $p\in\mathbb{R}^2$, let $h_p(x) = \indct{x = p}$. Let $\H = \{h_p: p\in\mathbb{R}^2\}$. $\H$ clearly has a VC-dimension of 1, but we claim that it is not tolerantly certifiable. We construct an argument similar to Observation~\ref{obs:lower-bound}. The idea is that no matter what queries the certifier chooses, we can always set $p$ to be a point that was not queried and is either inside $\U$ or outside $\V$ depending on the certifier's answer.
\end{proof}

\subsection{Witness Sets for Certification}

A common observation in the previous examples was that if the certifier could identify a set of points whose labels determined the points in $S$ that were in the margin of $h$, then querying those points was enough for robust certification. This motivates the following definition.

\begin{definition}[Witness sets]\label{def:wintness}
Given a hypothesis class $\H$ and a perturbation type $\U$, for any point $x\in X$, we say that $w(x)\subset X$ is a \emph{witness set} for $x$ if there exists a mapping $f: \{0, 1\}^{w(x)}\rightarrow\{0,1\}$ such that for any hypothesis $h\in \H$, $f(h|_{w(x)})=1$  if and only if $x$ lies in the margin of $h$ (where $h|_{w(x)}$ denotes the restriction of hypothesis $h$ to set $w(x)$).
\end{definition}

Clearly, all positive examples above were created using witness sets. The following theorem identifies a large class of $\H, \U$ pairs that exhibit finite witness sets.

\begin{theorem}\label{thm:partial_order}
For any $x\in X$, consider two partial orderings $\prec_0^x$ and $\prec_1^x$ over the elements of $\H$ where for $h_1, h_2\in \H$, we say $h_1\prec_1^x h_2$ if $\U(x)\cap h_1\subset\U(x)\cap h_2$, and $h_1\prec_0^x h_2$ if $\U(x)\setminus h_1\subset\U(x)\setminus h_2$. For both partial orderings we identify (as equivalent) hypotheses where these intersections co-incide and further we remove all hypotheses where the intersections are empty.
\footnote{Here, we think of hypotheses $h_1$ and $h_2$ as the pre-image of 1 (as noted in Section~\ref{s:setup}), and hence subsets of $X$.}
If both partial orders have finite number of minima for each $x$, then the pair $\H, \U$ exhibits a finite witness set and hence is certifiable.
\end{theorem}
\begin{proof}
For this proof we will identify hypotheses with their equivalence classes in each partial ordering.
Let $\Mcal_0(x)\subset\H$ be the set of minima for $\prec_0^x$ and $\Mcal_1(x)\subset\H$ for $\prec_1^x$. 
For each $h\in\Mcal_0(x)$, we pick a point $x'\in\U(x)$ such that $x'\in h$ but $x'\notin h'$ for any $h'\succ h$, thus forming a set $w_0(x)$. Similarly, we define the set $w_1(x)$. We claim that $w(x) = w_0(x)\cup w_1(x)\cup\{x\}$ is a witness set for $x$, i.e., we can determine whether $x$ is in the margin of any hypothesis $h\in\H$ by looking at labels that $h$ assigns to points in $w(x)$.

We only consider the case where $h(x) = 0$, since the $h(x) = 1$ case is similar. We claim that $x$ is in the margin of $h$ if and only if there exists a point in $w_1(x)$ that is assigned the label 1 by $h$. Indeed, suppose there exists such a point. Then since the point lies in $\U(x)$ and is assigned the opposite label as $x$ by $h$, $x$ must lie in the margin of $h$. For the other direction, suppose $x$ lies in the margin of $h$. Then there must exist a point $x'\in\U(x)$ such that $h(x') = 1$, which means there must be a hypothesis $\hat{h}\in\Mcal_1(x)$ such that $\hat{h}\prec_1^x h$, which means there must exist $\hat{x}\in w_1(x)$ such that $h(\hat{x}) = 1$.
\end{proof}

We can easily verify, for example, that for the $(\H, \U)$ pair defined in Observation~\ref{obs:half-spaces}, the set of minima defined by the partial orderings above is finite. Indeed, the (equivalence class of) half-spaces corresponding to the four corners of the unit cube constitute the minima. 

\subsection{Query complexity of adversarial attacks and its connection to robust PAC learning}
Even though in the literature on (practical) adversarial attacks an adversary is often modelled as an actual algorithm, in the theoretical literature the focus has been on whether adversarial examples merely exist\footnote{E.g., the definition of adversarial loss in Section~\ref{s:setup} is only concerned with whether an adversarial point exists.}. However, one can say a robust learner is successful if it merely finds a hypothesis that is potentially non-robust in the conventional sense yet whose adversarial examples are hard to find for the adversary. To formalize this idea, one needs to define some notion of ``bounded adversary'' in a way that enables the study of the complexity of finding adversarial examples. Attempts have been made at studying computationally bounded adversaries in certain scenarios~\cite{Garg2019arxiv} but not in the distribution-free setting. Here we study an adversary's \emph{query complexity}. We start by formally defining an adversary and discussing a few different properties of adversaries.

\begin{definition}\label{def:perfect_adversary}
For an $(\H, \U)$ pair, an \emph{adversary} is an algorithm $\A$ tasked with the following: given a set $S$ of $n$ points from the domain, and query access to a hypothesis $h\in\H$, return a set $S'$ such that (i) each point $x'\in S'$ is an adversarial point to some point in $S$ and (ii) for every $x\in S$ that has an adversarial point, there exists $x'\in S'$ such that $x'$ is an adversarial point for $x$. If the two conditions hold we call $S'$ an \emph{admissible attack} on $S$ w.r.t. $(\H, \U)$.
We call an adversary \emph{perfect} for $(\H, \U)$ if for every $S\subset X$ it outputs an admissible attack on $S$.
We say that the adversary is \emph{proper} if all its queries are in the set $\bigcup_{x\in S}\U(x)$. 
\end{definition}

There have been (successful) attempts~\cite{papernot2017practical,brendel2018decision} at attacking trained neural network models where the adversary was not given any information about the gradients, and had to rely solely on black-box queries to the model. Our definition of the adversary fits those scenarios. Next, we define the query complexity of the adversary.

\begin{definition}\label{def:efficient_adversary}
If, for $(\H, \U)$, there is a function $f:\naturals\to \naturals$ such that, for any $h\in \H$ and any set $S$, the adversary $\A$ will produce an admissible attack $S'$ after at most $f(|S|)$ queries, we say that adversary $\A$ has \emph{query complexity} bounded by $f$ on $(\H, \U)$. We say that the adversary is \emph{efficient} if $f(n)$ is linear in $n$.
\end{definition}

Note that it is possible that the adversary's queries are adaptive, i.e., the $i^{\text{th}}$ point it queries depends on the output of its first $i-1$ queries. A weaker version of an adversary is one where that is not the case.
\begin{definition}\label{def:non-adaptive_adversary}
An adversary is called \emph{non-adaptive} if the set of points it queries is uniquely determined by the set $S$ before making any queries to $h$.
\end{definition}

Intuitively, there is a connection between perfect adversaries and witness sets because a witness set merely helps identify the points in $S$ that have adversarial points, whereas an adversary finds those adversarial points. 
\begin{observation}\label{obs:witness_vs_non-adaptive_adversary}
If the $(\H, \U)$ pair exhibits a perfect, non-adaptive adversary with query complexity $f(n)$, then it also has witness sets of size $f(n)$.
\end{observation}

Finally, we tie everything together by showing that, for \emph{properly compressible classes}, the existence of a proper, perfect adversary implies that the robust learning problem has a small sample complexity. We say a class $\H$ is properly compressible if \emph{(i)} it admits a sample compression scheme~\cite{littlestone1986relating} of size $O(\vc\cdot \log(n))$---where $\vc$ is the VC-dimension of $\H$ and $n$ is the number of samples that are being compressed---, and \emph{(ii)} the hypothesis outputted by the scheme is always a member of $\H$. Note that (i) holds for all hypothesis classes (by a boosting-based compression scheme~\cite{schapire2013boosting,moran2016sample}). Furthermore, many natural classes are shown to have proper compression schemes, and it is open if this is true for all $\vc$-classes. (see \cite{floyd1995sample, ben1998combinatorial}).

\begin{theorem}\label{thm:bounded_adv_impl_learnable}
Assume $\H$ is properly compressible. If the robust learning problem defined by $(\H, \U)$ has a perfect, proper, and efficient adversary, then in the robust realizable-case ($\rLo{\U}{P}(\H)=0$) it can be robustly learned with $O(t\log^2(t))$ samples, where $t=\vc(\H)/\epsilon^2$.
\end{theorem}

\begin{proof}[Proof Sketch.]
We adapt the compression-based approach of~\cite{MontasserHS19} to prove the result. Let us assume that we are given a sample $S$ of size $|S|=m$ that is labeled by some $h\in \H$, and want to ``compress'' this sample using a small subset $K\subseteq S$. Let us assume that $\hat{h}$ is the hypothesis that is reconstructed using $K$. For the compression to succeed, we need to have $\rlo{\U}(h,x,y)=\rlo{\U}(\hat{h},x,y)$ for every $(x,y)\in S$. Given the perfect proper efficient adversary, we can find all the adversarial points in $S$ using $C$ queries per point in $S$, for some constant $C\geq 0$.
In fact, we can amend the points corresponding to these queries to $S$ to create an inflated set, which we call $T$. Note that $|T| \leq C\cdot |S|$. 
Furthermore, we can now replace the condition $\forall (x,y)\in S, \rlo{\U}(h,x,y)=\rlo{\U}(\hat{h},x,y) $ with $\forall (x,y)\in T, \blo(h,x,y)=\blo(\hat{h},x,y)$  (the latter implies the former because of the definition of perfect adversary). Therefore, our task becomes compressing $T$ with respect to the standard binary loss, for which we will invoke the assumption that $\H$ has a proper compression scheme. Assume this compression scheme compressed $T$ to a subset $K\subset T$ of size $k$. We argue that, by invoking the adversary we can convert the compressed set to only include points form the original sample $S$, and some additional bits as side information. For each point $x\in K$ in the compressed set that is an original sample point from $S$, we know that $x\in\U(x_S)$ for some $x_S\in S$, since the adversary is proper. We construct a new compressed set $K'\subset S$, by replacing such points $x$ with their corresponding points $x_S$ and bits $b$ to encode the rank of the point $x$ among the queries that the adversary would make for $x_S$. Now, the decompressor can first recover the set $K$ by invoking the adversary, and then use the standard decompression.
Finally, the size of the compression is $O(\log(m))VC(\H)$ and the results follows from the classic connection of compression and learning~\cite{littlestone1986relating,MontasserHS19}.
\end{proof}

This result shows an interesting connection between the difficulty of finding adversarial examples and that of robust learning. In particular, if the adversarial points can be found easily (at least when measured by query complexity), then robust learning is almost as easy as non-robust learning (in the sense of agnostic sample complexity). Or, stated in the contrapositive, if robust learning is hard, then even if adversarial points exist, finding them is going to be hard.
It is possible to further extend the result to the agnostic learning scenario, using the same reduction from agnostic learning to realizable learning that was proposed by \cite{david2016supervised} and used in \cite{MontasserHS19}.

\section{Conclusion}
We formalized the problem of black-box certification and its relation to an adversary with bounded query budget. We showed the existence of an adversary with small query complexity implies small sample complexity for robust learning. 
This suggests that the apparent hardness of robust learning -- compared to standard PAC learning -- in terms of sample complexity may not actually matter as long as we are dealing with bounded adversaries.
It would be interesting to explore other types of adversaries (e.g., non-proper and/or non-perfect) to see if they lead to efficient robust learners as well. Another interesting direction is finding scenarios where finite unlabeled data can substitute the knowledge of the marginal distribution discussed in Section~\ref{s:vc_classes}.

\section*{Acknowledgements} We thank the Vector Institute for providing us with the meeting space in which this work was developed! Ruth Urner and Hassan Ashtiani were supported by NSERC Discovery Grants.

\bibliography{refs}

\begin{thebibliography}{46}
\providecommand{\natexlab}[1]{#1}
\providecommand{\url}[1]{\texttt{#1}}
\expandafter\ifx\csname urlstyle\endcsname\relax
  \providecommand{\doi}[1]{doi: #1}\else
  \providecommand{\doi}{doi: \begingroup \urlstyle{rm}\Url}\fi

\bibitem[Akhtar \& Mian(2018)Akhtar and Mian]{akhtar2018threat}
Akhtar, N. and Mian, A.
\newblock Threat of adversarial attacks on deep learning in computer vision: A
  survey.
\newblock \emph{IEEE Access}, 6:\penalty0 14410--14430, 2018.

\bibitem[Alayrac et~al.(2019)Alayrac, Uesato, Huang, Fawzi, Stanforth, and
  Kohli]{stanforth2019labels}
Alayrac, J., Uesato, J., Huang, P., Fawzi, A., Stanforth, R., and Kohli, P.
\newblock Are labels required for improving adversarial robustness?
\newblock In \emph{Advances in Neural Information Processing Systems 32,
  {NeurIPS}}, pp.\  12192--12202, 2019.

\bibitem[Attias et~al.(2019)Attias, Kontorovich, and
  Mansour]{attias2018improved}
Attias, I., Kontorovich, A., and Mansour, Y.
\newblock Improved generalization bounds for robust learning.
\newblock In \emph{Algorithmic Learning Theory, {ALT}}, pp.\  162--183, 2019.

\bibitem[Awasthi et~al.(2019)Awasthi, Dutta, and
  Vijayaraghavan]{awasthi2019robustness}
Awasthi, P., Dutta, A., and Vijayaraghavan, A.
\newblock On robustness to adversarial examples and polynomial optimization.
\newblock In \emph{Advances in Neural Information Processing Systems,
  {NeurIPS}}, pp.\  13760--13770, 2019.

\bibitem[Ben{-}David et~al.(1995)Ben{-}David, Itai, and
  Kushilevitz]{Ben-DavidIK95}
Ben{-}David, S., Itai, A., and Kushilevitz, E.
\newblock Learning by distances.
\newblock \emph{Inf. Comput.}, 117\penalty0 (2):\penalty0 240--250, 1995.

\bibitem[Blumer et~al.(1989)Blumer, Ehrenfeucht, Haussler, and
  Warmuth]{blumer1989learnability}
Blumer, A., Ehrenfeucht, A., Haussler, D., and Warmuth, M.~K.
\newblock Learnability and the vapnik-chervonenkis dimension.
\newblock \emph{Journal of the ACM (JACM)}, 36\penalty0 (4):\penalty0 929--965,
  1989.

\bibitem[Brendel et~al.(2018)Brendel, Rauber, and Bethge]{brendel2018decision}
Brendel, W., Rauber, J., and Bethge, M.
\newblock Decision-based adversarial attacks: Reliable attacks against
  black-box machine learning models.
\newblock 2018.

\bibitem[Bubeck et~al.(2019)Bubeck, Lee, Price, and Razenshteyn]{BubeckLPR19}
Bubeck, S., Lee, Y.~T., Price, E., and Razenshteyn, I.~P.
\newblock Adversarial examples from computational constraints.
\newblock In \emph{Proceedings of the 36th International Conference on Machine
  Learning, {ICML}}, pp.\  831--840, 2019.

\bibitem[Carmon et~al.(2019)Carmon, Raghunathan, Schmidt, Duchi, and
  Liang]{CarmonRSDL19}
Carmon, Y., Raghunathan, A., Schmidt, L., Duchi, J.~C., and Liang, P.
\newblock Unlabeled data improves adversarial robustness.
\newblock In \emph{Advances in Neural Information Processing Systems 32,
  {NeurIPS}}, pp.\  11190--11201, 2019.

\bibitem[Chakraborty et~al.(2018)Chakraborty, Alam, Dey, Chattopadhyay, and
  Mukhopadhyay]{ChakrabortySurvey2018arxiv}
Chakraborty, A., Alam, M., Dey, V., Chattopadhyay, A., and Mukhopadhyay, D.
\newblock Adversarial attacks and defences: {A} survey.
\newblock \emph{CoRR}, abs/1810.00069, 2018.

\bibitem[Chen et~al.(2017)Chen, Zhang, Sharma, Yi, and Hsieh]{chen2017zoo}
Chen, P.-Y., Zhang, H., Sharma, Y., Yi, J., and Hsieh, C.-J.
\newblock Zoo: Zeroth order optimization based black-box attacks to deep neural
  networks without training substitute models.
\newblock In \emph{Proceedings of the 10th ACM Workshop on Artificial
  Intelligence and Security}, pp.\  15--26, 2017.

\bibitem[Cohen et~al.(2019)Cohen, Rosenfeld, and Kolter]{CohenRK19}
Cohen, J.~M., Rosenfeld, E., and Kolter, J.~Z.
\newblock Certified adversarial robustness via randomized smoothing.
\newblock In \emph{Proceedings of the 36th International Conference on Machine
  Learning, {ICML}}, pp.\  1310--1320, 2019.

\bibitem[Cullina et~al.(2018)Cullina, Bhagoji, and Mittal]{cullina2018pac}
Cullina, D., Bhagoji, A.~N., and Mittal, P.
\newblock Pac-learning in the presence of adversaries.
\newblock In \emph{Advances in Neural Information Processing Systems,
  {NeurIPS}}, pp.\  230--241, 2018.

\bibitem[David et~al.(2016)David, Moran, and Yehudayoff]{david2016supervised}
David, O., Moran, S., and Yehudayoff, A.
\newblock Supervised learning through the lens of compression.
\newblock In \emph{Advances in Neural Information Processing Systems, {NIPS}},
  pp.\  2784--2792, 2016.

\bibitem[Diochnos et~al.(2018)Diochnos, Mahloujifar, and
  Mahmoody]{diochnos2018adversarial}
Diochnos, D., Mahloujifar, S., and Mahmoody, M.
\newblock Adversarial risk and robustness: General definitions and implications
  for the uniform distribution.
\newblock In \emph{Advances in Neural Information Processing Systems 31,
  {NeurIPS}}, pp.\  10359--10368, 2018.

\bibitem[Diochnos et~al.(2019)Diochnos, Mahloujifar, and
  Mahmoody]{diochnos2019arxiv}
Diochnos, D.~I., Mahloujifar, S., and Mahmoody, M.
\newblock Lower bounds for adversarially robust {PAC} learning.
\newblock \emph{CoRR}, abs/1906.05815, 2019.

\bibitem[Dong et~al.(2018)Dong, Liao, Pang, Su, Zhu, Hu, and
  Li]{dong2018boosting}
Dong, Y., Liao, F., Pang, T., Su, H., Zhu, J., Hu, X., and Li, J.
\newblock Boosting adversarial attacks with momentum.
\newblock In \emph{Proceedings of the IEEE conference on computer vision and
  pattern recognition}, pp.\  9185--9193, 2018.

\bibitem[Feige et~al.(2015)Feige, Mansour, and Schapire]{feige2015learning}
Feige, U., Mansour, Y., and Schapire, R.
\newblock Learning and inference in the presence of corrupted inputs.
\newblock In \emph{Conference on Learning Theory, {COLT}}, pp.\  637--657,
  2015.

\bibitem[Feldman(2017)]{Feldman17}
Feldman, V.
\newblock A general characterization of the statistical query complexity.
\newblock In \emph{Proceedings of the 30th Conference on Learning Theory,
  {COLT}}, pp.\  785--830, 2017.

\bibitem[Garg et~al.(2019)Garg, Jha, Mahloujifar, and Mahmoody]{Garg2019arxiv}
Garg, S., Jha, S., Mahloujifar, S., and Mahmoody, M.
\newblock Adversarially robust learning could leverage computational hardness.
\newblock \emph{CoRR}, abs/1905.11564, 2019.

\bibitem[Goodfellow et~al.(2018)Goodfellow, McDaniel, and
  Papernot]{GoodfellowMP18}
Goodfellow, I.~J., McDaniel, P.~D., and Papernot, N.
\newblock Making machine learning robust against adversarial inputs.
\newblock \emph{Commun. {ACM}}, 61\penalty0 (7):\penalty0 56--66, 2018.

\bibitem[G{\"{o}}pfert et~al.(2019)G{\"{o}}pfert, Ben{-}David, Bousquet, Gelly,
  Tolstikhin, and Urner]{GopfertBBGTU19}
G{\"{o}}pfert, C., Ben{-}David, S., Bousquet, O., Gelly, S., Tolstikhin, I.~O.,
  and Urner, R.
\newblock When can unlabeled data improve the learning rate?
\newblock In \emph{Conference on Learning Theory, {COLT}}, pp.\  1500--1518,
  2019.

\bibitem[Gourdeau et~al.(2019)Gourdeau, Kanade, Kwiatkowska, and
  Worrell]{GourdeauKK019}
Gourdeau, P., Kanade, V., Kwiatkowska, M., and Worrell, J.
\newblock On the hardness of robust classification.
\newblock In \emph{Advances in Neural Information Processing Systems 32,
  {NeurIPS}}, pp.\  7444--7453, 2019.

\bibitem[Haussler \& Welzl(1987)Haussler and Welzl]{HausslerW87}
Haussler, D. and Welzl, E.
\newblock epsilon-nets and simplex range queries.
\newblock \emph{Discret. Comput. Geom.}, 2:\penalty0 127--151, 1987.

\bibitem[Kearns(1998)]{Kearns98}
Kearns, M.~J.
\newblock Efficient noise-tolerant learning from statistical queries.
\newblock \emph{J. {ACM}}, 45\penalty0 (6):\penalty0 983--1006, 1998.

\bibitem[Littlestone \& Warmuth(1986)Littlestone and
  Warmuth]{littlestone1986relating}
Littlestone, N. and Warmuth, M.
\newblock Relating data compression and learnability.
\newblock 1986.

\bibitem[Madry et~al.(2018)Madry, Makelov, Schmidt, Tsipras, and
  Vladu]{madry2017towards}
Madry, A., Makelov, A., Schmidt, L., Tsipras, D., and Vladu, A.
\newblock Towards deep learning models resistant to adversarial attacks.
\newblock In \emph{6th International Conference on Learning Representations,
  {ICLR}}, 2018.

\bibitem[Montasser et~al.(2019)Montasser, Hanneke, and Srebro]{MontasserHS19}
Montasser, O., Hanneke, S., and Srebro, N.
\newblock {VC} classes are adversarially robustly learnable, but only
  improperly.
\newblock In \emph{Conference on Learning Theory, {COLT}}, pp.\  2512--2530,
  2019.

\bibitem[Montasser et~al.(2020)Montasser, Goel, Diakonikolas, and
  Srebro]{montasser2020efficiently}
Montasser, O., Goel, S., Diakonikolas, I., and Srebro, N.
\newblock Efficiently learning adversarially robust halfspaces with noise.
\newblock \emph{arXiv preprint arXiv:2005.07652}, 2020.

\bibitem[Moran \& Yehudayoff(2016)Moran and Yehudayoff]{moran2016sample}
Moran, S. and Yehudayoff, A.
\newblock Sample compression schemes for vc classes.
\newblock \emph{Journal of the ACM (JACM)}, 63\penalty0 (3):\penalty0 1--10,
  2016.

\bibitem[Narodytska \& Kasiviswanathan(2017)Narodytska and
  Kasiviswanathan]{narodytska2017simple}
Narodytska, N. and Kasiviswanathan, S.
\newblock Simple black-box adversarial attacks on deep neural networks.
\newblock In \emph{2017 IEEE Conference on Computer Vision and Pattern
  Recognition Workshops (CVPRW)}, pp.\  1310--1318. IEEE, 2017.

\bibitem[Papernot et~al.(2017)Papernot, McDaniel, Goodfellow, Jha, Celik, and
  Swami]{papernot2017practical}
Papernot, N., McDaniel, P., Goodfellow, I., Jha, S., Celik, Z.~B., and Swami,
  A.
\newblock Practical black-box attacks against machine learning.
\newblock In \emph{Proceedings of the 2017 ACM on Asia conference on computer
  and communications security}, pp.\  506--519, 2017.

\bibitem[Salman et~al.(2019)Salman, Li, Razenshteyn, Zhang, Zhang, Bubeck, and
  Yang]{SalmanLRZZBY19}
Salman, H., Li, J., Razenshteyn, I.~P., Zhang, P., Zhang, H., Bubeck, S., and
  Yang, G.
\newblock Provably robust deep learning via adversarially trained smoothed
  classifiers.
\newblock In \emph{Advances in Neural Information Processing Systems 32,
  {NeurIPS}}, pp.\  11289--11300, 2019.

\bibitem[Schapire(1990)]{schapire1990strength}
Schapire, R.~E.
\newblock The strength of weak learnability.
\newblock \emph{Machine learning}, 5\penalty0 (2):\penalty0 197--227, 1990.

\bibitem[Schapire \& Freund(2013)Schapire and Freund]{schapire2013boosting}
Schapire, R.~E. and Freund, Y.
\newblock Boosting: Foundations and algorithms.
\newblock \emph{Kybernetes}, 2013.

\bibitem[Schmidt et~al.(2018)Schmidt, Santurkar, Tsipras, Talwar, and
  Madry]{schmidt2018adversarially}
Schmidt, L., Santurkar, S., Tsipras, D., Talwar, K., and Madry, A.
\newblock Adversarially robust generalization requires more data.
\newblock In \emph{Advances in Neural Information Processing Systems,
  {NeurIPS}}, pp.\  5014--5026, 2018.

\bibitem[Shalev-Shwartz \& Ben-David(2014)Shalev-Shwartz and
  Ben-David]{shalev2014understanding}
Shalev-Shwartz, S. and Ben-David, S.
\newblock \emph{Understanding Machine Learning: From Theory to Algorithms}.
\newblock Cambridge University Press, 2014.

\bibitem[Sinha et~al.(2018)Sinha, Namkoong, and Duchi]{SinhaND18}
Sinha, A., Namkoong, H., and Duchi, J.~C.
\newblock Certifying some distributional robustness with principled adversarial
  training.
\newblock In \emph{6th International Conference on Learning Representations,
  {ICLR}}, 2018.

\bibitem[Su et~al.(2019)Su, Vargas, and Sakurai]{su2019one}
Su, J., Vargas, D.~V., and Sakurai, K.
\newblock One pixel attack for fooling deep neural networks.
\newblock \emph{IEEE Transactions on Evolutionary Computation}, 23\penalty0
  (5):\penalty0 828--841, 2019.

\bibitem[Szegedy et~al.(2014)Szegedy, Zaremba, Sutskever, Bruna, Erhan,
  Goodfellow, and Fergus]{SzegedyZSBEGF13}
Szegedy, C., Zaremba, W., Sutskever, I., Bruna, J., Erhan, D., Goodfellow,
  I.~J., and Fergus, R.
\newblock Intriguing properties of neural networks.
\newblock In \emph{2nd International Conference on Learning Representations,
  {ICLR}}, 2014.

\bibitem[Valiant(1984)]{Valiant84}
Valiant, L.~G.
\newblock A theory of the learnable.
\newblock \emph{Commun. {ACM}}, 27\penalty0 (11):\penalty0 1134--1142, 1984.

\bibitem[Vapnik \& Chervonenkis(1971)Vapnik and Chervonenkis]{vapnikcherv71}
Vapnik, V.~N. and Chervonenkis, A.~Y.
\newblock On the uniform convergence of relative frequencies of events to their
  probabilities.
\newblock \emph{Theory of Probability \& Its Applications}, 16\penalty0
  (2):\penalty0 264--280, 1971.

\bibitem[Wang et~al.(2018)Wang, Jha, and Chaudhuri]{WangJC18}
Wang, Y., Jha, S., and Chaudhuri, K.
\newblock Analyzing the robustness of nearest neighbors to adversarial
  examples.
\newblock In \emph{Proceedings of the 35th International Conference on Machine
  Learning, {ICML}}, pp.\  5120--5129, 2018.

\bibitem[Wong \& Kolter(2018)Wong and Kolter]{WongK18}
Wong, E. and Kolter, J.~Z.
\newblock Provable defenses against adversarial examples via the convex outer
  adversarial polytope.
\newblock In \emph{Proceedings of the 35th International Conference on Machine
  Learning, {ICML}}, pp.\  5283--5292, 2018.

\bibitem[Yang et~al.(2019)Yang, Rashtchian, Wang, and
  Chaudhuri]{Kamalika2019arxiv}
Yang, Y., Rashtchian, C., Wang, Y., and Chaudhuri, K.
\newblock Adversarial examples for non-parametric methods: Attacks, defenses
  and large sample limits.
\newblock \emph{CoRR}, abs/1906.03310, 2019.

\bibitem[Yin et~al.(2019)Yin, Ramchandran, and Bartlett]{YinRB19}
Yin, D., Ramchandran, K., and Bartlett, P.~L.
\newblock Rademacher complexity for adversarially robust generalization.
\newblock In \emph{Proceedings of the 36th International Conference on Machine
  Learning,{ICML}}, pp.\  7085--7094, 2019.

\end{thebibliography}
\bibliographystyle{icml2020}


\newpage
\appendix
\newcommand{\M}{{\mathcal M}}

\section{Note on our notation for sets and functions}\label{app:notation}
We use the following notation for sets and functions:

\begin{tabular}{ll}
 $2^X$ &  the power-set (set of all subsets) of $X$\\
 $Y^X$ &   the set of all functions from $X$ to $Y$\\
 $f: X \to Y$ &  $f$ is a function from $X$ to $Y$\\
\end{tabular}

Functions from some set $X$ to some set $Y$ are a special type of relations between $X$ and $Y$.
Thus a function $f: X \to Y$ is a subset of $X\times Y$, namely 
\[
 f = \{(x,y)\in X\times Y ~\mid~ y = f(x)\} 
\]
If $h: X\to Y$ is a (not necessarily binary) classifier, and $P$ is a probability distribution over $X\times Y$, then the probability of misclassification is $P(\err{h})$, where $\err{h}$ is the complement of $h$ in $X\times Y$, that is
\[
 \err{h} = \{(x,z)\in X\times Y ~\mid~ z \neq f(x)\} = (X\times Y)\setminus h
\]
If $Y = \{0,1\}$ is a binary label space, then it is also common to identify classifiers $h: X \to \{0,1\}$ with a subset of the domain, namely the set $h^{-1}(1)$, that is the set of points that is mapped to label $1$ under $h$:
\[
h^{-1}(1) = \{ x \in X ~\mid~ h(x) = 1\}
\]
We switch between identifying $h$ with $h^{-1}(1)$ and viewing $h$ as a subset of $X\times Y$, depending on which view aids the simplicity of argument in a given context.

We defined the margin areas of a classifier (with respect to a perturbation type) again as subsets of $X\times Y$.
\[
\mar{\U}{h} =  \{(x,y)\in X\times Y ~\mid~ \exists z\in\U(x): h(x)\neq h(z)\}
\]
Note, that here, if for a given domain point $x$, we have $(x,y) \in \mar{\U}{h}$ for some $y\in Y$, then $(x,y') \in \mar{\U}{h}$ for all $y'\in Y$.
Thus, the sets $\mar{\U}{h}\subseteq X\times Y $ are not functions.
Rather, they can naturally be identified with their projection on $X$, and we again do so if convenient in the context.

The given definitions of $\err{h}$ and $\mar{\U}{h}$,  naturally let us express the robust loss as the probably measure of a subset of $X\times Y$:
$$\rLo{\U}{P} (h) = P(\err{h}\cup\mar{\U}{h}).$$

\section{Note on measurability}\label{app:measurable}
Here, we note that allowing the perturbation type $\U$ to be an arbitrary mapping from the domain $X$ to $2^X$ can easily lead to the adversarial loss being not measurable, even if $\U(x)$ is a measurable set for every $x$.  Consider the case $X = \reals$, and a distribution $P$ with $P_X$ uniform on the interval $[0,2]$. Consider a subset $M\subseteq (0,1)$ that is not Borel-measurable. Consider a simple threshold function
\[
 f: \reals \to \{0,1\}, \quad f(x) = \indct{x < 1}
\]
and a the following perturbation type:
\[
 \U(x) ~= ~\left\{\begin{array}{lll}
                                         \emptyset & \text{if} & x\notin M\\
                                         \{x+1\} & \text{if }& x\in M
                                        \end{array}
 \right.
\]
Clearly, $f$ is a measurable function, and every set $\U(x)$ is measurable.
However, we get $\mar{\U}{f} = M $, that is, the margin area of $f$ under these perturbations is not measurable, and therefore the adversarial loss with respect to $\U$ is not measurable.
Note that the same phenomenon can occur for sets $\U$ that are always open intervals containing the point $x$. With the same function $f$, for perturbation sets
\[
 \U(x) ~= ~\left\{\begin{array}{lll}
                                         \B_r(x) \cap (0,1) & \text{if} & x <1, x\notin M \\
                                         \B_r(x) \cap (1,2) & \text{if} & x > 1 \\
                                         (0,2) & \text{if }& x\in M \text{ or } x = 1
                                        \end{array}
 \right.
\]
we get $\mar{\U}{f} = M \cup \{1\}$, which again is not measurable.

We may thus make the following implicit assumptions on the sets $\U(x)$:
\begin{itemize}
 \item $x\in \U(x)$ for all $x\in X$
 \item if $X$ is an uncountable domain, we assume $X$ is equipped with a separable metric and $\U(x) = \B_r(x)$ is an open ball around $x$  
\end{itemize}
Note that the latter assumption implies that $\mar{\U}{h}$ is measurable for a measurable predictor $h$. This can be seen as follows:
It $h$ is a (Borel-)measurable function, then both $h^{-1}(1) = \{x\in X \mid h(x) = 1\}$ and $h^{-1}(0) = \{x\in X \mid h(x) = 0\}$ are measurable sets by definition.
Now, if we consider  ``blowing up'' these sets by adding open balls around each of their members, we obtain open (as a union of open sets), and thus measurable sets:
\[
 \M_r^1 : =  \bigcup_{x \in h^{-1}(1)} \B_r(x)
\]
and 
\[
 \M_r^0 : =  \bigcup_{x \in h^{-1}(0)} \B_r(x).
\]
Now the margin area can be expressed as a simple union of intersections, and is therefore also measurable:
\[
 \mar{\U}{h} = (\M_r^1 \cap h^{-1}(0)) \cup (\M_r^0 \cap h^{-1}(1))
\]
Note that this equality depends on the balls as perturbation sets inducing a symmetric relation, that is $x\in \U(z)$ if and only if $z\in \U(x)$. This condition does not hold in the above counterexample construction. However, this argument shows it is sufficient (together with openness) for measurability of the sets $\mar{\U}{h}$.
\section{Proofs and additional results to Section 3}\label{app:ssl}

\subsection{Some background}
We first briefly recall the notions of $\epsilon$-nets and $\epsilon$-approximations and their role in learning binary hypothesis classes of finite VC-dimension.
We will frequently use these concepts in our proofs in this section.

\paragraph{$\epsilon$-nets and $\epsilon$-approximations \citep{HausslerW87}} Let $Z$ be some domain set and let $\G\subseteq 2^Z$ be a collection of (measurable) subsets of $Z$ and let $D$ be a probability distribution over $Z$.
Let $\epsilon \in (0,1)$. A finite set $S\subseteq Z$ is an \emph{$\epsilon$-net} for $\G$ with respect to $D$ if
\[
 S\cap G \neq \emptyset
\]
for all $G\in \G$ with $P(G) \geq \epsilon$. That is, an $\epsilon$-net ``hits'' every set in the collection $\G$ that has probability weight at least $\epsilon$.
A finite set $S\subseteq Z$ is an \emph{$\epsilon$-approximation} for $\G$ with respect to $D$ if
\[
 \left|P(G) - \frac{|G \cap S|}{|S|}\right|\leq \epsilon
\]
for all $G\in\G$.
It is well known that,
given also  $\delta\in (0,1)$, if $\G$ has finite VC-dimension, then an iid sample $S$ of size at least $\tilde{\Theta}\left(\frac{\vc(\G) + \log(1/\delta)}{\epsilon}\right)$   from distribution $D$ is an $\epsilon$-net for $\G$ with probability at least $(1-\delta)$ (see, eg, Theorem 28.3 in \cite{shalev2014understanding}); and an iid sample $S$ of size at least $\tilde{\Theta}\left(\frac{\vc(\G) + \log(1/\delta)}{\epsilon^2}\right)$   from distribution $D$ is an $\epsilon$-approximation for $\G$ with probability at least $(1-\delta)$ (we are omitting logarithmic factors here). 

\paragraph{Learning VC-classes (\cite{vapnikcherv71,Valiant84,blumer1989learnability}} If $X$ is a domain, $Y= \{0,1\}$ is a binary label space, and $\H\subseteq Y^X\subseteq 2^{(X\times Y)}$ is a hypothesis class of finite VC-dimension, then the class of error sets $\err{\H} = \{\err{h} \mid h\in \H\}$, that is the class of complements of $\H$, has finite VC-dimension $\vc(\err{\H}) = \vc(\H)$. 
For distributions $P$ over $X\times Y$, we get that sufficiently large samples (as indicated above) are $\epsilon$-nets of $\err{\H}$.
Now, if a sample $S$ is an $\epsilon$-net of the  class $\err{\H}$ with respect to $P$, then every function in the \emph{version space} $\V_S(\H)$ of $S$ with respect to $\H$ has error less than $\epsilon$.
Recall the version space is defined as those functions in $\H$ that have zero error on the points in $S$, that is
\[
 \V_S(\H) = \{h\in \H \mid \bLo{S}(h) = 0\}.
\]
If $P$ is realizable by $\H$, an empirical risk minimizing (ERM) learner, will output a hypothesis from the version space (the version space is non-empty under the realizability assumption) and therefore output a predictor of binary loss at most $\epsilon$ (with high probability).

For general (not necessarily realizable) learning, note that large enough samples $S$ are $\epsilon$-approximation of $\err{\H}$ (with high probability at least $1-\delta$ as above).
This is also referred to as \emph{uniform convergence} for the hypothesis class $\H$.
Thus, every function $h\in \H$ has true loss that is $\epsilon$-close to its empirical loss on $h$, and any empirical risk minimizer is a successful learner for $\H$ even in the agnostic case.

With these preparations, we proceed to the proofs of Theorem \ref{thm:learnability_pos}, Theorem \ref{thm:robust_realizable} and Theorem \ref{thm:realizable_clusterable}.

\subsection{Proofs}\label{app:proofs}

\begin{proof}[Proof of Theorem \ref{thm:learnability_pos}]
We recall that the robust loss of a classifier $h$ with respect to distribution $P$ over $X\times Y$ is given by 
\[
 \rLo{\U}{P}(h) ~=~ P(\err{h}\cup\mar{\U}{h})
\]
Thus, to show that empirical risk minimization with respect to the robust loss is a successful learner, we need to guarantee that large enough samples are $\epsilon$-approximations for the class 
$\G = \{(\err{h}\cup\mar{\U}{h}) \subseteq X \times Y ~\mid~ h\in \H\}$ of point-wise unions error and margin regions. 

A simple counting argument involving Sauer's Lemma (see Chapter 6 in \cite{shalev2014understanding}, and exercises therein)
shows that $\vc(\G) \leq 2D\log(D)$, where $D = \vc(\H) + \vc(\mH)$.
Thus, a sample of size $\tilde{\Theta}\left(\frac{D\log D + \log(1/\delta) }{\epsilon^2}\right)$ will be an $\epsilon$-approximation of $\G$ with respect to $P$ with probability at least $1-\delta$ over the sample. 
Thus any empirical risk minimizer with respect to $\rlo{\U}$ is a successful proper and agnostic robust learner for $\H$.
\end{proof}

\begin{proof}[Proof of Theorem \ref{thm:robust_realizable}]
Note that robust realizability means there exists a $h^*\in \H $ with $\rLo{\U}{P}(h^*) = 0$ and this implies $\bLo{P}(h^*) = 0$. 
That is, the distribution is (standard) realizable by $\H$. 
The above outlined VC-theory tells us that for an iid sample $S$ of size $\Tilde{\Theta}\left(\frac{\vc(\H)+ \log(\frac{1}{\delta})}{\epsilon}\right)$ guarantees that all functions in the \emph{version space} of $S$ (that is all $h\in \H$ with $\Lo{S}(h) = 0$) have true binary loss at most $\epsilon$ (with probability at least $1-\delta$). 
Now, with access to $P_X$ a learner can remove all hypotheses with $P(\mar{\U}{h}) > 0$ from the version space and return any remaining hypothesis.
Note that, since $h^*$ is assumed to satisfy $\rLo{\U}{P}(h^*) = 0$, we have $P(\err{h^*}) = 0$ and $P(\mar{\U}{h}) = 0$, therefore, the pruned version will contain at least one function.
Now, for any function $h_p$ in the the pruned version space, we obtain
\begin{align*}
 \rLo{\U}{P}(h_p) & ~=~ P(\err{h_p} ~\cup~ \mar{\U}{h_p}) \\
  & ~\leq~ P(\err{h_p}) + P(\mar{\U}{h_p}) \\
  & ~\leq~ \epsilon + 0 ~=~ \epsilon.
\end{align*}
Thus, access to the marginal allows for a successful learner in the robust-realizable case.
\end{proof}

\begin{proof}[Proof of Theorem \ref{thm:double}]
We will modify the lower bound construction of Theorem \ref{thm:stat_impossibility} as follows: we add an additional point $x_8$ to the domain set, which has zero probability mass under both $P^1$ and $P^2$. 
We set $\U(x_8) = \U(x_7) = \{x_7, x_8\}$.
We modify the probability weights of points $x_1, \ldots, x_6$ under $P^1$ and $P^2$ by dividing them by $2$ (i.e., all respective denominators in the proof of Theorem \ref{thm:stat_impossibility} become $12$, and we add weight accordingly to $x_7$, so that $P^i(x_7) = 1/2 + 1/12$ under both distributions. Functions $h_1$ and $h_2$ are extended to the new point by setting $h_1(x_8) = h_2(x_8) = 0$. Thus, the indistinguishability phenomenon of the construction remains the same.

Now we add a function $h_r = \indct{x = x_8}$ to the class $\H$.
This yields $P^{i}(\err{h_r}) = P^{i}(x_8) = 0$, for $i\in\{1,2\}$, thus both distributions are realizable with respect to the $0/1$-loss now. However $P^{i}(\mar{\U}{h_r}) = 1/2 + 1/12$, for $i\in\{1,2\}$, thus $h_r$ has adversarial loss $1/2 + 1/12$ on both distribution and the construction thus remains otherwise analogous. We now have $\rLo{\U}{P^i}(h_i) = 2/12$, thus $h_r$ is does not affect the optimal robust classifier in $\H$.

Additionally, we add the constant $1$ function $h_c$ to the class $\H$. For this function (as for any constant classifier) the margin area is empty, thus the distributions are ``margin realizable'' by $\H$. However, we have $P^{i}(\err{h_c}) = 1$, for $i\in\{1,2\}$, thus $h_c$  also has adversarial loss $1$ on both distribution and the construction still remains otherwise unchanged.
\end{proof}

\subsection{Additional results}\label{app:add_results}

\subsubsection{0/1-Realizability\\
$\exists h^*\in \H$ with $\bLo{P}(h^*) = 0$}

Theorem \ref{thm:double} shows that $0/1$-realizability does not suffice for semi-supervised learning with a margin oracle for $\H$. However, here we show that the following \emph{extended margin oracle} does suffice: we assume that the learner has oracle access to the weights of the sets $\mar{\U}{h}, h\Delta h'$, and $\mar{\U}{h} \cap (h\Delta h')$, for all $h,h'\in \H$, where the sets $h\Delta h' \subseteq X$ are defined as follows:
\[
h\Delta h' = \{x\in X ~\mid~ h(x) \neq h'(x)\}.
\]

\begin{theorem}\label{thm:realizable}
Let $X$ be some domain, $\H$ a hypothesis class with finite VC-dimension and $\U:X\to 2^X$ any perturbation type.
If a learner is given additional access to an extended margin oracle for $\H$, then $\H$ is properly learnable with respect to the robust loss $\rlo{\U}$ and the class of distributions $P$ that are $0/1$-realizable by $\H$, that is we have $\bLo{P}(\H) = 0$, with labeled sample complexity $\tilde{O}(\frac{\vc(\H) + \log(1/\delta)}{\epsilon})$.
\end{theorem}

\begin{proof}
As in the proof of Theorem \ref{thm:robust_realizable}, since we assume the distribution to be $0/1$-realizable by $\H$, the version space of a labeled sample of the given size will include only functions with (true) binary loss at most $\epsilon$. The learner can choose a function $h_e$ from this version space. Now, given the extended margin oracle, the learner can choose a function $h_r$ that minimizes the robust loss with respect to labeling function $h_e$. That is, the extended margin oracle allows to find the minimizer in $\H$ of the robust loss on a distribution $(P_X, h_e)$, that shares the marginal with the data generating distribution $P$, but labels domain points according to $h_e$.

Let $h^* \in \H$ be a function with $\bLo{P}(h^*) = 0$. Thus, we can identify the distribution $P$ with $(P_X, h^*)$.
Now we first show that for any classifier $h$, the difference between its robust loss  with respect to $P =(P_X, h^*)$ and with respect to $(P_X, h_e)$ is bounded by $\epsilon$.

Let $h\in\H$ be given. Then we have
\begin{align*}
\rLo{\U}{P}(h) & ~=~ \rLo{\U}{(P_X, h^*)}(h)\\
& ~=~ P_X(\mar{\U}{h}\cup (h^*\Delta h))\\    
& ~=~ P_X(\mar{\U}{h}) +  P_X((h^*\Delta h)\setminus \mar{\U}{h})
\end{align*}
and
\begin{align*}
\rLo{\U}{P_X,h_e}(h) & ~=~ P_X(\mar{\U}{h}\cup (h_e\Delta h))\\    
& ~=~ P_X(\mar{\U}{h}) +  P_X((h_e\Delta h)\setminus \mar{\U}{h}).
\end{align*}
Thus, we get 
\begin{align*}
& |\rLo{\U}{P}(h) - \rLo{\U}{P,h_e}(h)|\\
 ~\leq~ & |P((h^*\Delta h)\setminus \mar{\U}{h}) - P_X((h_e\Delta h)\setminus \mar{\U}{h})|\\
~\leq~ & |P((h^*\Delta h)\setminus \mar{\U}{h}) - (P_X((h_e\Delta h^*)\setminus \mar{\U}{h})\\
& \quad + P_X((h^*\Delta h)\setminus \mar{\U}{h}))|\\ 
~\leq~ & |P((h_e\Delta h^*)\setminus \mar{\U}{h}) |\\
~\leq~ & P((h_e\Delta h^*) ~\leq~ \epsilon.
\end{align*}
where the second inequality follows from 
\[
(h_e\Delta h) \subseteq (h_e\Delta h^*) \cup (h^*\Delta h),
\]
and thus
\[
(h_e\Delta h)\setminus \mar{\U}{h}  \subseteq ((h_e\Delta h^*)\setminus \mar{\U}{h}) \cup ((h^*\Delta h)\setminus \mar{\U}{h}).
\]
Note that $|\rLo{\U}{P}(h) - \rLo{\U}{P,h_e}(h)|\leq \epsilon$ for all $h\in\H$ implies that we also have:
\[
| \inf_{h\in \H} \rLo{\U}{P}(h) - \inf_{h\in \H} \rLo{\U}{P_X,h_e}(h)| ~\leq~\epsilon 
\]
Thus, for the output $h_r$ of the above procedure, we get
\begin{align*}
\rLo{\U}{P}(h_r) & ~\leq~ \rLo{\U}{P_X, h_e}(h_r) + \epsilon\\
& ~=~ \inf_{h \in \H} \rLo{\U}{P_X,h_e}(h) + \epsilon\\
& ~\leq~ \inf_{h \in \H} \rLo{\U}{P}(h) + 2 \epsilon\\
\end{align*}
Substituting $\epsilon/2$ for $\epsilon$ in this argument completes the proof.
\end{proof}

\subsubsection{$0/1$-Realizability on a $\U$-clusterable task: 
$\exists h^*\in \H$ with $\bLo{P}(h^*) = 0$ and $\exists f^*\in \F$ with  $\rLo{\U}{P}(f^*) = 0$}

We start by observing that the existence of an $f^*\in \F$ with  $\rLo{\U}{P}(f^*) = 0$ implies that the support of $P_X$ is sitting on $\U$-separated clusters. Note that we do not assume that the perturbation type $\U$ induces a symmetric relation; we can nevertheless consider the clusters as connected components of a directed graph where we place a directed edge between two domain instances $x$ and $x'$ if and only if $x$ is in the support of $P_X$ and $x'\in\U(x)$. The assumption $\rLo{\U}{P}(f^*) = 0$ then implies that these clusters are label-homogeneous. This observation leads to a simple, yet improper learning scheme for the robust loss. 

We show that, if the distribution is also $0/1$-realizable by $\H$, a leaner that knows that marginal, can return a hypothesis with robust loss at most $\epsilon$.
We note that here, the learner does not return a hypothesis from the class $\H$. In return, the guarantee is stronger in the sense that the robust loss of the returned classifier is close to the overall (among all binary predictors, rather than just those in $\H$) best achievable robust loss.

\begin{theorem}\label{thm:realizable_clusterable}
Let $X$ be some domain, $\H$ a hypothesis class with finite VC-dimension and $\U:X\to 2^X$ any perturbation type.
If a learner has access to a labeled sample of size
\[
\tilde{O}\left(\frac{\vc(\H) + \log{1/\delta}}{\epsilon} \right)
\]
and, additionall has access to $P_X$, then the class $\F$ of all binary predictors is learnable with respect to the robust loss $\rlo{\U}$ and the class of distributions $P$ that are realizable by $\H$ (that is, $\bLo{P}(\H) = 0$) and robust realizable with respect to $\F$ (that is, $\rLo{\U}{P}(\F) = 0$).
\end{theorem}

\begin{proof}
Recall that, to avoid measurability issues, we either assume a countable domain, or, in case of an uncountable domain, that the perturbation sets are open balls with respect to some separable metric.
The arguments below hold for both cases.

We now start by observing that the existence of an $f^*\in \F$ with  $\rLo{\U}{P}(f^*) = 0$ implies that the support of $P_X$ is sitting on $\U$-separated clusters.
Note that (in the case of a countable domain) we do not assume that the perturbation type $\U$ induces a symmetric relation.
We derive the clusters as follows: we define a (directed)
graph on $X$, where we place an edge between from domain elements $x$ to $x'$ if and only if $x$ is in the support of $P_X$ and $x'\in\U(x)$.
We now let $\C\subseteq 2^X$ be the collection of connected components of the induced undirected graph.
Since $\rLo{\U}{P}(f^*) = 0$, thus $P(\mar{\U}{f^*}) = 0$, the function $f^*$ is label homogeneous on these clusters (except, potentially, for subsets of $P_X$-measure $0$, and we may then identify $f^*$ with a function that is label homogeneous on the clusters). 

Now, since $P$ is $\H$-realizable, there is an $h^*\in \H$ with $\bLo{P}(h^*) = 0$. Note that $h^*$ is not necessarily label homogeneous on the clusters (since $h^*$ may have a positive robust loss, that is it may be the case that $P(\mar{\U}{h^*}) >0$). However, $h^*$ agrees with $f^*$ on the support of $P_X$ 

(except on a set with measure $0$), since both functions have zero binary loss, $\bLo{P}(h^*) = \bLo{P}(f^*) = 0$. 
Let $\supp(P_X)$ denote the support of $P_X$. That is, for any cluster $C\in\C$, $h^*$ is label-homogeneous (and in agreement with $f^*$) on the subset $C\cap\supp(P_X)$.

Note that, since we assume knowledge of the marginal, we may assume that a learner knows the collection of clusters $\C$ and the support of $P_X$. 
We now define a learning scheme as follows.

As in the proof of Theorem \ref{thm:robust_realizable}, due to the $\H$-realizability ($\bLo{P}(\H) = 0$), we know that with high probability over a large enough sample $S$, all functions $h\in \V_{S}(\H)$ in the version space satisfy $\bLo{P}(h)\leq \epsilon$. 
Moreover, due to the $\H$-realizability, there will exist functions (for example $h^*$) in the version space that label the intersections $C \cap \supp(P_X)$ of the clusters in with the support  of $P_X$ homogeneously. 
Thus, employing the knowledge of $P_X$, the learner can prune the version space by removing all functions from the version space that don't label all sets $C \cap \supp(P_X)$ homogeneously, and pick a function $h_p$ from this pruned version space.

Now the learner can construct a new classifier $f_p$, that agrees with $h_p$ on the sets $C\cap \supp(P_X)$ and labels the full clusters homogeneously, that is, if $x\in C\cap \supp(P_X)$ for some cluster $C\in \C$, then we set $f_p(x') = h_p(x)$ for all $x'\in C$.
Now, by construction of $f_p$ (recall the definition of the clusters), we get $P(\mar{\U}{f_p}) = 0$.
Moreover, we have $P(\err{f_p}) \leq \epsilon$ (inherited from $h_p$ since $h_p$ and $f_p$ agree on the support of $P_X$).
Thus 
$$\rLo{\U}{P}(f_p) ~\leq~ \epsilon ~\leq~ \rLo{\U}{P}(\H) + \epsilon,$$ 
which is what we needed to show.
\end{proof}


\section{Proof from Section \ref{s:query}}\label{app:query_proofs}

\begin{proof}[Proof of Observation \ref{obs:lower-bound}]
We prove this statement for the case when the certifier is restricted to be deterministic, and leave the proof of the probabilistic case to future work. Suppose the entire data distribution is concentrated on one point, and wlog suppose the point is the origin and has label 1. Let $B$ be the unit ball centred at the origin. Thus the certifier's task is to determine if $h$ passes through $B$ or not. We construct a scheme for answering the certifier's queries in a way so that no matter what sequence of queries it chooses to ask, once it commits to a verdict, we can find a halfspace that is consistent with the answers we provided to the queries, but inconsistent with the certifier's verdict.

It is easier to work in a dual space using a standard duality argument, where the dual of a point $(a, b)$ is the line $ax + by + 1 = 0$ and vice versa. This duality transform has the following two useful properties: 1) a point is to the left of a line if and only if the dual of the point is to the left of the dual of the line, and 2) a point is inside the unit ball if and only if its dual does not intersect the unit ball. Thus in the dual space, the certifier picks a line and asks whether the hidden point is to its left or right, and needs to determine if the hidden point is inside the unit ball or not. Our strategy, then, is to consider the arrangement of lines created by the certifier's queries thus far, and locate a cell that contains a part of $B$'s circumference. We answer the certifier's query as if the point was inside this cell. This cell will have a non-zero volume whenever the certifier stops, and we can select a point inside the cell that is inside or outside $B$ depending on the certifier's answer. That we can always find such a cell can be seen with an argument using induction. For the base case, there are no queries and hence no lines. Thus the entire place is such a cell. Suppose we have identified such a cell after seeing $m$ lines. If the next line does not pass through the cell it still satisfies the property in question. If the next line does pass through the cell, it divides the cell into two smaller cells one of which will satisfy the property.
\end{proof}


\section{Proof of Theorem \ref{thm:bounded_adv_impl_learnable}}

We start by providing the definition of proper sample compression for adversarially robust learning.

\begin{definition}[Adversarially Robust Proper Compression]
We say $(\H, \U)$ admits robust proper compression of size $k$ if there exist a (decoder) function $\phi: (X\times Y)^k \rightarrow \H$ such that the following holds: for every $h\in \H$ and every $S_X \subset X$, there exist $K_X \subset S_X$ such that 

$$\forall x\in S_X, \rlo{\U}(h, x, h(x)) = \rlo{\U}(\phi(K), x, h(x))$$

where $K$ is the labeled version of $K_X$ (labeled by $h$).
\end{definition}

Note that in the above definition, $k=|K|$ can potentially depend on the size of the set, $m=|S_X|$. However, this dependence should be sub-linear (e.g., logarithmic) to later result in a non-vacuous sample complexity upper bound. The following theorem draws the connection between compression and robust learning.

\begin{theorem}\label{thm:Compression_implies_learning} If $(\H, \U)$ admits an adversarially robust proper compression of size $k$, then the sample complexity of robust learning of $(\H, \U)$ in the robustly realizable setting is $O(k\log(k/\epsilon)/\epsilon^2)$.
\end{theorem}
\begin{proof}
This theorem can be proved in a similar way to that of classical (non-robust) sample compression proposed by \cite{littlestone1986relating}. For the proof in the context of robust compression we refer the reader to Lemma 11 in \cite{MontasserHS19}. Note that the hypothesis returned by the decoder of the compression scheme has to have zero robust loss on all of the samples (due to robust realizability).
\end{proof}

In order to proceed, we need to show that for properly compressible classes, the existence of a perfect proper efficient adversary means that a small-sized robust proper compression scheme exists.

\begin{theorem}\label{thm:Compression_scheme_efficient_adversary} Let $\H$ be any properly (non-robustly) compressible class. Assume $(\H, \U)$ has a perfect proper adversary with query complexity $O(m)$. Then $(\H, \U)$ admits a robust proper compression of size $O(VC(\H)\log (m))$.
\end{theorem}

Let us postpone the proof of Theorem~\ref{thm:Compression_scheme_efficient_adversary} for now and complete the proof of Theorem~\ref{thm:bounded_adv_impl_learnable}.
\begin{proof}[Proof of Theorem~\ref{thm:bounded_adv_impl_learnable}.] Assume that $(\H, \U)$ has a perfect, proper, and efficient adversary. Based on Theorem~\ref{thm:Compression_scheme_efficient_adversary}, we conclude that $(\H, \U)$ admits a robust proper compression scheme of size $O(VC(\H)\log (m))$. We can now use Theorem~\ref{thm:Compression_implies_learning} to bound the sample complexity of learning. In particular, it will be enough to have $m>\Omega(k\log(k/\epsilon)/\epsilon^2)$ where $k=\Theta(VC(\H)\log (m))$. Therefore, it will suffice to have $m=\Omega(VC(\H)\log^2(VC(\H)/\epsilon)/\epsilon^2)$.
\end{proof}

Therefore, it only remains to construct a robust proper compression scheme and prove Theorem~\ref{thm:Compression_scheme_efficient_adversary}. We denote by $S_X$ the unlabeled portion of the sample $S$. 

\begin{proof}[Proof of Theorem~\ref{thm:Compression_scheme_efficient_adversary}]

Recall that we want to show that there exists $K_X \subset S_X$ such that 

$$\forall x\in S_X, \rlo{\U}(h, x, h(x)) = \rlo{\U}(\phi(K), x, h(x))$$

where $K$ is the labeled version of $K_X$ (labeled by $h$). We know that $(\H, \U)$ has a perfect adversary with query complexity $O(m)$. Let $Q_S$ be the set of queries that the adversary asks on $S$ to find the adversarial points (so $|Q_S|=O(m)$). Let $Q$ be the labeled version of $Q_S$ (i.e., each query with its answer from $h$). We claim that for a proper compression to succeed it will be enough to have

\begin{equation}
 \label{eq:compression_requirement}
 \forall z\in S_X\cup Q_S,  \phi(K)\big|_z=h\big|_z
\end{equation}

The reason is that if the two hypotheses from $\H$ have the same behaviour on $T=S_X\cup Q_S$ then they should have the same robust loss on $S_X$ as well (otherwise the adversary would not be perfect). The final step is to come up with a proper compression scheme that satisfies (\ref{eq:compression_requirement}).

Let $T_Y$ be the labeled version of $T$ that is labeled by $h$. Recall that $\H$ is a properly (non-robustly) compressible class. Therefore, $T_Y$ can be properly (non-robustly) compressed into a set $I\subseteq T_Y$ such that $|T_Y|=O(VC(\H)\log(|T_Y|))$. The catch is that $I$ may contain points that are outside of $S$, and therefore we cannot simply use $I$ for robust proper compression. We can modify the compression scheme by adding some additional bits of information so that its output contains only points from $S$. For any $x\in S$, let $Q_x\subseteq \Ucal(x)$ be the set of points that the adversary queries to attack $x$. Note that $|Q_x|=O(1)$ due to the efficiency of the adversary. We replace any $(x,y)\in S\setminus I$ with $(x_0,y_0)$ where $(x_0,y_0)\in S$ and $x\in \Ucal(x_0)$. Also, we use a constant number of bits to encode the labels of $Q_{x_0}$ and also the subset of $Q_{x_0}$ that was chosen by the non-robust compression scheme. The decoder works as follows. Given $(x_0,y_0)$, it can simulate the adversary on $x_0$ (using the bits that represent the labels) to recover $Q_{x_0}$. It can use the other part of the bits to recover the subset of ${Q_{x_0}}$ that was present in $I$ (let us call this set $G_x$). 
Finally, it would run the decoder of the proper non-robust compression scheme on $\cup_{x\in S} G_x$.

\end{proof}

\end{document}